\documentclass{article}

\usepackage{microtype}
\usepackage{graphicx}
\usepackage{booktabs} %

\usepackage{hyperref}

\usepackage{amsmath}
\usepackage{amsthm}
\usepackage{algorithm}
\usepackage{algorithmic}
\usepackage{amsfonts}
\usepackage{amssymb}
\usepackage[utf8]{inputenc}

\usepackage{makecell}

\usepackage{thmtools}
\usepackage{thm-restate}

\usepackage{xcolor}
\usepackage{cases} 
\usepackage{nicefrac}
\DeclareMathOperator*{\argmax}{arg\,max}

\usepackage{hyperref}
\usepackage{cleveref}

\usepackage{wrapfig}
\usepackage{subfig}
\usepackage{multirow}
\usepackage{multicol}
\usepackage{natbib}

\def\X{{\mathcal X}}
\def\H{{\mathcal H}}
\def\R{{\mathbb R}}
\def\E{{\mathbb E}}
\def\B{{\mathbb B}}

\usepackage[mathscr]{euscript}
 \let\mathscr\relax%
\usepackage[scr]{rsfso}
\newcommand{\powerset}{\raisebox{.15\baselineskip}{\Large\ensuremath{\wp}}}

\usepackage[accepted]{icml2021}

\icmltitlerunning{Connecting Interpretability and Robustness in Decision Trees through Separation}

\begin{document}

\twocolumn[
\icmltitle{Connecting Interpretability and Robustness in Decision Trees through Separation}

\icmlsetsymbol{equal}{*}

\begin{icmlauthorlist}
\icmlauthor{Michal Moshkovitz}{ucsd}
\icmlauthor{Yao-Yuan Yang}{ucsd}
\icmlauthor{Kamalika Chaudhuri}{ucsd}
\end{icmlauthorlist}

\icmlaffiliation{ucsd}{University of California, San Diego}

\icmlkeywords{machine learning, adversarial robustness, interpretable}

\vskip 0.3in
]

\printAffiliationsAndNotice{\icmlEqualContribution} %

\begin{abstract}
Recent research has recognized interpretability and robustness as essential properties of trustworthy classification.
    Curiously, a connection between robustness and interpretability was empirically observed, but the theoretical reasoning behind it remained elusive. In this paper, we rigorously investigate this connection. Specifically, we focus on interpretation using decision trees and robustness to $l_{\infty}$-perturbation. Previous works defined the notion of $r$-separation as a sufficient condition for robustness. 
    We prove upper and lower bounds on the tree size in case  the data is
    $r$-separated. We then show that a tighter bound on the size is possible
    when the data is linearly separated. We provide the first algorithm with
    provable guarantees both on robustness, interpretability, and accuracy in
    the context of decision trees. Experiments confirm that our algorithm yields
    classifiers that are both interpretable and robust and have high accuracy.
    The code for the experiments is available at \url{https://github.com/yangarbiter/interpretable-robust-trees}.
\end{abstract}

\section{Introduction}
Deploying machine learning (ML) models in high-stakes fields like healthcare, transportation, and law, requires the ML models to be trustworthy.
 Essential ingredients of trustworthy models are  explainability and robustness: if we do not understand the reasons for the model's prediction, we cannot trust the model; if small changes in the input modifies the model's prediction, we cannot trust the model. 
Previous works hypothesized that there is a strong connection between robustness and explainability. They empirically observed that robust models lead to better explanations \cite{chen2019robust,ross2017improving}.  
In this work, we take a rigorous approach towards understanding the connection between robustness and interpretability.

We focus on binary predictions, where each example has $d$ features and the label of each example is in $\{-1,+1\}$, so an ML model is a hypothesis $f:\mathbb{R}^d\rightarrow\{-1,1\}.$ We want our model to be (i) robust to adversarial $\ell_\infty$ perturbations, i.e., for a small distortion, $\|\delta\|_\infty$, the model's response is similar, $f(x)=f(x+\delta)$, for most examples $x$, (ii) interpretable, i.e., the model itself is simple and so self-explanatory, and (iii) have high-accuracy. A common type of interpretable models are decision trees \cite{molnar2020interpretable}, which we call \emph{tree-based explanation} and focus on in this paper.

Prior literature \cite{yang2020adversarial}
showed that data \emph{separation} is a sufficient condition for a robust and accurate classifier.
A data is $r$-separated if the distance between the two closest examples with different labels is at least $2r$. Intuitively, if $r$ is large, then the data is well-separated. A separated data guarantees that points with opposite labels are far from each other, which is essential to construct a robust model. 

In this paper, we examine whether separation implies tree-based explanation. We first show that for a decision tree to have accuracy strictly above $1/2$ (i.e., better than random), the data must be bounded. Henceforth we assume that the data is in $[-1,1]^d.$ We start with a trivial algorithm that constructs a tree-based explanation with complexity (i.e., tree size) $2^{O(d/r)}$. For constant $r$, we show a matching lower bound of $2^{\Omega(d)}.$ Thus, we have a matching lower and upper bound on the explanation size of $2^{\Theta(d)}$. Thus, separation implies robustness and interpretability. Unfortunately, for a large number of features, $d$, the explanation size is too high to be useful in practice.

In this paper, we show that designing a simpler explanation is possible with a stronger separability assumption --- linear separability with a $\gamma$-margin.
This assumption was recently used to gain a better understanding of neural networks \cite{soudry2018implicit,nacson2019convergence,shamir2020gradient}.
More formally, this assumption means that there is a vector $w$ with $\|w\|=1$ such that $yw\cdot x\geq\gamma$ for each labeled example $(x,y)\in\mathbb{R}^d\times\{-1,1\}$ in the data \cite{shalev2014understanding}. Our main building block is a proof that, under the linearity assumption, there is always at least one feature that provides non-trivial information for the prediction.
To formalize this, we use the known notion of \emph{weak learners} \cite{kearns1988learning}, which  guarantees the existence of hypothesis with accuracy bounded below by more than $\nicefrac12$. 

Our weak-learnability theorem, together with \citet{kearns1999boosting}, implies that a popular CART-type algorithm \cite{breiman1984classification} is guaranteed to give a decision tree with size $1/\epsilon^{O(1/\gamma^2)}$ and accuracy $1-\epsilon$. Therefore, under the linearity assumption, we can design a tree with complexity independent of the number of features. Thus, even if the number of features, $d$, is large, the interpretation complexity is not affected. This achieves our first goal of constructing an interpretable model with provable guarantees.

Recently, several research papers gave a theoretical justification for CART's empirical success \cite{brutzkus2020id3,brutzkus2019optimality,blanc2019top,blanc2020provable,fiat2004decision}.
Those papers assume that the underlying distribution is uniform or features chosen independently. For many cases, this assumption does not hold. For example, in medical data, there is a strong correlation between age and different diseases. 
 On the other hand, we give a theoretical justification for CART without resorting to the feature-independence assumption. We use, instead, the linear separability assumption. We believe that our method will allow, in the future, proofs with less restrictive assumptions.

So far, we showed how to construct an interpretable model, but we want a model that will not be just interpretable but also robust. Decision trees are not robust by-default  \cite{chen2019robust}. For example, a small change at the feature at the root of the decision tree leads to an entirely different model (and thus to entirely different predictions): the model defined by the left subtree and the model defined by the right subtree. We are left with the questions: can robustness and interpretability co-exist? How to construct a model that is guaranteed to be both interpretable and robust? To design such a model, we focus on a specific kind of decision tree --- risk score \cite{ustun2017optimized}. A risk score is composed of several conditions (e.g., $age \geq 75$) and each matched with a weight, i.e., a small integer. A score $s(x)$ of an example $x$ is the weighted sum of all the satisfied conditions. The label is then a function of the score $s(x)$. A risk score is a specific case of decision trees, wherein at each level in the tree, the same feature is queried. The number of parameters required to represent a risk score is much smaller than their corresponding decision trees, hence they might be considered more interpretable than decision trees \cite{ustun2017optimized}.

We design a new learning algorithm, BBM-RS, for learning risk scores that rely on the Boost-by-Majority (BBM) algorithm \cite{freund1995boosting} and our weak learner theorem. It yields a risk score of size $O(\gamma^{-2}\log(1/\epsilon))$ and accuracy $1-\epsilon$. Thus, we found an algorithm that creates a risk score with provable guarantees on size and accuracy. As a side effect, note that BBM allows to control the interpretation complexity easily. Importantly, we show that risk scores are also guaranteed to be robust to $\ell_\infty$ perturbations, by deliberately adding a small noise to dataset (but not too much noise to make sure that the noisy dataset is still linearly separable). Therefore, we design a model that is guaranteed to have high accuracy and be both interpretable and robust, achieving our final goal. 

Finally, in Section~\ref{sec:experiment}, we test the validity of the separability assumption and the quality of the new algorithm on real-world datasets that were used previously in tree-based explanation research. On most of the datasets, less than $12\%$ points were removed to achieve an $r$-separation with $r\geq0.05$. For comparison, for binary feature-values $\{-1,1\}$, and $\ell_\infty$ distance, the best value for $r$ is $r=1$. The added percentage of points required to be removed for the dataset to be linearly separable is less than $7\%$ on average. Thus, we observe that real datasets are close to being separable and even linearly separable.  
In the next step, we explored the quality of our new algorithm, BBM-RS. Even though it has provable guarantees only if the data is linearly separable, we run it on real datasets that do not satisfy this property. 
We compare BBM-RS to different algorithms for learning: decision trees \cite{quinlan1986induction}, small risk scores \cite{ustun2017optimized}, and robust decision trees \cite{chen2019robust}. All algorithms try to maximize accuracy, but different algorithms try to, additionally, minimize interpretation complexity or maximize robustness.
None of the algorithms aimed to optimize both  interpretability and robustness. We compared the (i) interpretation complexity, (ii) robustness, and (iii) accuracy of all four algorithms. We find that our algorithm provides a model with better interpretation complexity and  robustness while having comparable accuracy. 

To summarize, our main contributions are:

 \textbf{Interpretability under separability: optimal bounds.} We show lower and upper bounds on decision tree size for $r$-separable data with $r=\Theta(1)$, of $2^{\Theta(d)}$. Namely, our upper bound proves that for any separable data, there is a tree of size $2^{O(d)}$, and the lower bound proves that separability cannot guarantee an explanation smaller than $2^{\Omega(d)}$.
  
\textbf{Algorithm with provable guarantees on interpretability and robustness.} 
Designing algorithms that have provable guarantees both on interpretability,  robustness, and accuracy in the context of decision trees is highly sought-after, yet there was no such algorithm before our work.   
We design the first learning algorithm that has provable guarantees both on interpretability, robustness, and accuracy of the returned model, under the assumption that the data is linearly separable with a margin. 

While the CART algorithm is empirically highly effective, its theoretical analysis has been elusive for a long time. As a side effect, we provide an analysis of CART under the assumption of linear separability. To the best of our knowledge, this is the first proof with a distributional assumption that does not include feature independence.

    \textbf{Experiments.} We verify the validity of our assumptions empirically and show that for real datasets, if a small percentage of points is removed then we get a  linear separable dataset. We also compare our new algorithm to other algorithms that return interpretable models \cite{quinlan1986induction,ustun2017optimized,chen2019robust}  and show that if the goal is to design a model that is both interpretable and robust, then our method is preferable.

\section{Related Work}
\textbf{Post-hoc explanations.}
There are two main types of explanations: post hoc explanations \cite{ribeiro2016model} and intrinsic explanations \cite{rudin2019stop}. Algorithms for post hoc explanation take as an input a black-box model and return some form of explanation. Intrinsic explanations are simple models, so the models are self-explanatory. The main advantage of algorithms for post hoc explanations \cite{lundberg2017unified,lundberg2018consistent, ribeiro2016should,koh2017understanding,ribeiro2018anchors,deutch2019constraints,li2020quantitative,boer2020personal} is that they can be used on any model. However, they host a variety of problems: 
they introduce a new source of error stemming from the explanation method \cite{rudin2019stop}; they can be fooled \cite{slack2020fooling}; some explanations methods are not robust to common pre-processing steps \cite{kindermans2019reliability}, and can be independent both of the model and the data generating process \cite{adebayo2018sanity}. Because of the critics against post hoc explanations, in this paper, we focus on intrinsic explanations.

\textbf{Explainability and robustness.} 
Different works research the intersection of explanation and robustness of black-box models   
\cite{lakkaraju2020robust}, decision trees \cite{chen2019robust,andriushchenko2019provably}, and deep neural networks  \cite{szegedy2013intriguing,goodfellow2014explaining,madry2017towards,ross2017improving}. %
Unfortunately, the quality of these methods are only verified empirically. %
On the theoretical side, most works analyzed explainability and  robustness separately. Explainability was researched for supervised learning \cite{garreau2020looking, garreau2020explaining,mardaoui2020analysis,hu2019optimal} and unsupervised learning  \cite{moshkovitz2020explainable,frost2020exkmc,laber2021price}.  For robustness, \citet{cohen2019certified} showed that the technique of randomized smoothing has robustness guarantees. 
\citet{ignatiev2019relating}  connected adversarial examples and a different type of explainability from the point of view of formal logic.

\textbf{Risk scores.} 
Ustun and Rudin \cite{ustun2017optimized} designed a new algorithm for learning risk scores by solving an appropriate optimization problem. They focused on constructing an interpretable model with high accuracy and did not consider robustness, as we do in this work. 

\section{Preliminaries}
We investigate models that are (i) with high-accuracy, (ii) robust, and (iii) interpretable, as formalized next. 

\textbf{High accuracy.} We consider the task of binary classification over a domain $\X\subseteq\R^d$. Let $\mu$ be an underlying probability distribution\footnote{In the paper, we will assume that $\mu$ has additional properties, like separation or linear separation.} over labeled examples $\X\times\{-1,+1\}.$ 
 The input to a learning algorithm $\mathcal{A}$ consists of a labeled sample  $S \sim\mu^m$, and its output is a hypothesis $h:\X\rightarrow\{-1,+1\}.$
 The accuracy of $h$ is equal to $\Pr_{(x,y)\sim\mu}(h(x)= y)$. The sample complexity of $\mathcal{A}$ under the distribution $\mu$, denoted $m(\epsilon,\delta) : (0, 1)^2\rightarrow  \mathbb{N}$, is a function
mapping desired accuracy $\epsilon$ and confidence $\delta$ to the minimal positive integer $m(\epsilon,\delta)$ such that for any
$m \geq m(\epsilon,\delta)$, with probability at least $1-\delta$ over the drawn of an i.i.d. sample $S\sim\mu^m$,
the output $A(S)$ has accuracy of at least $1-\epsilon$.

\textbf{Robustness.}
We focus on the $\ell_\infty$ ball, $\B$, and denote the $r$-radius ball around a point $x\in\X$ as $\B(x,r)$.
A hypothesis $h:\X\rightarrow\{-1,+1\}$ is \emph{robust} at $x$ with radius $r$ if for all $x'\in \B(x,r)$ we have that $h(x)=h(x').$
In \cite{wang2018analyzing}, the notion of \emph{astuteness} was introduced to measure the robustness of a hypothesis $h$. The astuteness
of $h$ at radius $r > 0$ under a distribution $\mu$ %
is $$\Pr_{(x,y)\sim\mu}[\forall x'\in \B(x,r).\; h(x')=y].$$
For a hypothesis to have high astuteness the positive and negative examples need to be  separated. A binary labeled data is \emph{$r$-separated} if for every two labeled examples $(x^1,+1)$,$(x^2,-1)$, it holds that 
$\|x^1-x^2\|_\infty\geq 2r.$

\textbf{Interpretability.} We focus on intrinsic explanations, also called interpretable models \cite{Rudin19}, where the model itself is the explanation.  There are several types of interpretable models, e.g., logistic regression, decision rules, and anchors \cite{molnar2020interpretable}. One of the most fundamental interpretable models, which we focus on in this paper, is 
\emph{decision trees} \cite{quinlan1986induction}. In a decision tree, each leaf corresponds to a label, and each inner node corresponds to a threshold and a feature. The label of an example is the leaf's label of the corresponding path.  

In the paper we also focus on a specific type of decision trees, \emph{risk scores} \cite{ustun2019learning}, see Table~\ref{tab:risk_score}. It is defined by a series of $m$ conditions and a weight for each condition. Each condition compares one feature to a threshold, and the weights should be small integers. A score, $s(x)$, of an example $x$ is the number of satisfied conditions out of the $m$ conditions, each multiplied by the corresponding weight. 
The prediction of the risk model $f$ is the sign of the score, $f(x)=sign(s(x))$. 
A risk score can be viewed as a decision tree where at each level there is the same feature-threshold pair. Since the risk-score model has fewer parameters than the corresponding decision tree, it may be considered more interpretable.

\begin{table}[ht]
    \setlength{\tabcolsep}{1pt} 
    \small
    \centering
    \begin{tabular}{lcc|c}
    \toprule
    feature                          & \multicolumn{2}{c}{weights} & \\
                                     & LCPA & BBM-RS &  \\
    \midrule
    Bias term                     & -6 & -7 & + ... \\
    Age $\geq$ 75                    &  - &  2 & + ... \\
    Called in Q1                     &  1 &  2 & + ... \\
    Called in Q2                     & -1 &  - & + ... \\
    Called before                    &  1 &  4 & + ... \\
    Previous call was Successful     &  1 &  2 & + ... \\
    Employment variation rate $<-1$  &  5 &  4 & + ... \\
    Consumer price index $\geq 93.5$ &  1 &  - & + ... \\
    3 month euribor rate $\geq 200$  & -2 &  - & + ... \\
    3 month euribor rate $\geq 400$  &  5 &  - & + ... \\
    3 month euribor rate $\geq 500$  &  2 &  - & + ... \\
    \midrule
                            & & total score  = &     \\
    \bottomrule
    \end{tabular}
    \caption{
    Two risk score models: LCPA \cite{ustun2019learning} and our new BBM-RS algorithm on the bank dataset~\cite{moro2014data}.
    Each satisfied condition is multiplied by its weight and summed. Bias term is always satisfied. 
    If the total score $>0$, the risk model predicts ``1" (i.e., the client will open a bank account after a marketing call). All features are binary (either $0$ or $1$).}
    \label{tab:risk_score} 
\end{table}

\section{Separation and Interpretability}
We want to understand whether separation implies the existence of a small tree-based explanation. Our first observation is that the data has to be bounded for a tree-based explanation to exist. If the data is unbounded, then to achieve a training error slightly better than random, the tree size must depend on the size of the training data, see Section~\ref{subsec:general_bounded}, Theorem~\ref{thm:unbounded_lower_bound}. 

In Section~\ref{subsec:general_case} we investigate lower and upper bounds for decision tree's size, assuming separation. Specifically, in Theorem~\ref{thm:general_sep_implies_dt}, we show that if the data is bounded, in $[-1,1]^d$, then $r$-separability implies a tree based-explanation with tree depth $O(\nicefrac{d}{r})$. Importantly, the depth of the tree is independent of the training size, so a tree-based explanation exists. Nevertheless, even for a constant $r$, the size of the tree is exponential in $d$. 
In Theorem~\ref{thm:general_sep_lower_dt}, we show that this bound is tight as there is a $1$-separable dataset that requires an exponential size to achieve accuracy even negligibly better than random. The conclusion from this section is that if all we know is that the data is $r$-separability for constant $r$, we understand the interpretation complexity: it is exactly $2^{\Theta(d)}$. However, the explanation is of exponential size in $d$. In Section~\ref{sec:linearly_separable}, we improve the interpretation complexity by assuming a stronger separability assumption. We will assume linear separability with a margin. All proofs are in Section~\ref{apx:separation_and_explainability}.

\subsection{Bounded}\label{subsec:general_bounded}
In Theorem~\ref{thm:unbounded_lower_bound}, we show that the data has to be bounded for a small decision tree to exist.  
In fact, boundedness is necessarily, even if the data is constrained to be linearly separable. 
For any tree size $s$ and a given accuracy, we can construct a linearly-separable dataset such that any tree of size $s$ cannot have the desired accuracy.  

\begin{restatable}{thm}{unboundedLowerBound}
\label{thm:unbounded_lower_bound}
 For any tree size $s$ and $\gamma>0$, there is a dataset in $\R^2$ that is linearly separable, and any decision tree with size $s$ has accuracy less than $\frac12+\gamma.$
\end{restatable}

\subsection{Upper and lower bounds}\label{subsec:general_case}
Assuming the data in $[-1,1]^d$ is $r$-separated, Theorem~\ref{thm:general_sep_implies_dt} tells us that one can construct a decision tree with depth $6d/r$ and training error $0$ (and from standard VC-arguments also accuracy $1-\epsilon$, with enough examples) . Importantly, the depth of the tree is independent of the training size $n$. Nevertheless, it means the size of the tress is exponential in $d$. The idea of the proof is to fine-grain the data to bins of size about $r$, in each coordinate. From this construction, it is clear that the returned model is robust at any training data. 

\begin{restatable}{thm}{generalSepImpliesDt}
\label{thm:general_sep_implies_dt}
For any labeled data in $[-1,1]^d\times\{-1,1\}$ that is $r$-separated, there is a decision tree of depth at most $\frac{6d}{r}$ which has a training error $0$. 
\end{restatable}

Theorem~\ref{thm:general_sep_lower_dt} proves a matching lower bound by constructing a dataset such that any tree that achieves error better than random, the tree size must be exponential in $d$. 
The dataset proving this lower bound is parity. More specifically, it contains the points $\{-1,+1\}^d$ and the label of each point $x$ is the xor of all of its coordinates.

\begin{restatable}{thm}{generalSepLowerDt}
\label{thm:general_sep_lower_dt}
There is a labeled dataset in $[-1,1]^d$ which is $1$-separated and has the following property. For any $\gamma>0$ and any decision tree $T$ that achieves accuracy $0.5+\gamma$, the size of $T$ is at least $\gamma 2^{d}.$
\end{restatable}

\section{Linear Separability}\label{sec:linearly_separable}
In the previous section, we showed that $\Theta(1)$-separability implies a decision tree with size exponential in $d$, and we showed a matching lower bound. This section explores a stronger assumption than separability that will guarantee a smaller tree, i.e., a simpler explanation. This assumption is that the data is linearly separable with a margin. More formally, data is $\gamma$-linearly separable if there is $w\in\R^d$, $\|w\|_1=1$, such that for each positive example $x$ it holds that  $w\cdot x \geq \gamma$ and for each negative example $x$ it holds that $w\cdot x \leq -\gamma.$  
Note that without loss of generality $w_i\geq 0$ (if the inequality does not hold, multiply the $i$-th feature in each example by $-1$). Thus, we can  interpret $w$ as a distribution over all the features. 
One might interpret the highest $w_i$ as the most important feature. However, this will be misleading. For example, if all that data has the same value at the $i$-th feature, this feature is meaningless. 
In the experimental section,  Section~\ref{sec:experiment}, we show that real datasets are close to being linearly separable.

The key step in designing a model which is both interpretable and robust is to show that one feature can provide a nontrivial prediction. In particular, in Section \ref{subsec:weak_learner} we show that the hypotheses class, $\H_t=\{h_{i,\theta}\}$, is a weak learner, where 
 $$
  h_{i,\theta}(x) =
  \begin{cases}
    \begin{array}{lr}
        +1 & \text{if } x_i \geq \theta \\
        -1 & \text{o.w. }
        \end{array}
        \end{cases}
  $$
This class is similar to the known decision stumps class, but it does not contain hypotheses of the form ``if $x_i\leq\theta$ then $+1$ else $-1$". The reason will become  apparent in Section~\ref{subsec:robust}, but for now, we will hint that it helps achieve robustness. 

In Section~\ref{subsec:id3}, we observe that weak learnability immediately implies that the known CART algorithm constructs a tree of size independent of $d$ \cite{kearns1999boosting}. %
Unfortunately, decision trees are not necessarily robust. To overcome this difficulty, we focus on one type of decision trees, risk scores, which are interpretable models on their own. In Section~\ref{subsec:robust} we show how to use \cite{freund1995boosting} together with our weak learnability theorem to construct a risk score model. We also show that this model is robust. This concludes our quest of finding a model that is \emph{guaranteed} to be robust, interpretable, and have high-accuracy under the linearity separable assumption. In Section~\ref{sec:experiment} we will evaluate the model on several real datasets. 

\subsection{Weak learner}\label{subsec:weak_learner}
This section shows that under the linearity assumption, we can always find a feature that gives nontrivial information, which is formally defined using the concept of a \emph{weak learner} class. We say that a class $\H$ is a weak learner if for every function $f$ that is $\gamma$-linearly separable and for every distribution $\mu$ over the examples, there is hypothesis $h\in\H$ such that $\Pr_{x\sim \mu}(h(x)=f(x))$ is strictly bigger than $1/2$, preferably at least $1/2+\Omega(\gamma)$.     
Finding the best hypothesis in $\H_t$ can be done efficiently using dynamic programming \cite{shalev2014understanding}. The interesting question is how to prove that there must be a weak learner in $\H_t$. This will be the focus of the current section. 

One might suspect that if the data is linearly separable by the vector $w$ (i.e., for each labeled example $(x,y)$ it holds that $y wx\geq\gamma$), then $h_i$ which corresponds to the highest  $w_i$ is a weak learner. Conversely, if $w_i$ is small, then the corresponding hypotheses $h_i$ will have a low accuracy. These claims are not true. To illustrate this, think about the extreme example where $w_1=0$ but $x_1$ completely predicts the output of any example $x$. From the viewpoint of $w$, the first feature is irrelevant, as it does not contribute to the term $w\cdot x$, but the first feature is a perfect predictor.

To prove that there is always a hypothesis in $\H_t$ with accuracy $0.5+\Omega(\gamma)$, we view $\H_t$ as a graph. Namely, define a bipartite graph where the vertices are the examples and the hypotheses and there is an edge between a hypothesis $h$ and example $x$ if $h$ correctly predicts $x$. The edges of the graph are defined so that (i) the degree of the hypotheses vertices corresponds to its accuracy and (ii) the linearity assumption ensures that the degree of the example vertices is high. These two properties of the graph proves the theorem. All proofs are in Section~\ref{apx:weak_learner}.

\begin{restatable}{thm}{WeakLearner}
\label{thm:weak_learner}
Fix $\alpha>0$. For any data in $[-1,1]^d\times\{-1,1\}$ that is labeled by a  $\gamma$-linearly separable hypothesis $f$ and for any distribution $\mu$ on the examples, there is a hypothesis $h\in\H_t$  such that $$\Pr_{x\sim \mu}(h(x) = f(x)) \geq  \frac12+\frac{\gamma}{2}-\alpha.$$
\end{restatable}

As a side note, in \cite{shalev2008equivalence}, a different connection between linear separability and weak learning was formed. They view each example in the hypotheses basis, and on this basis, the famous minimax theorem implies that linearity is equivalent to weak learnability. In this paper, we focus on the case that the data, \emph{in its original form}, is linearly separable. Nonetheless, in the case where the features are binary, the two views, the original and hypotheses basis, coincide. %
However, our method also holds for the nonbinary case. %

So far, we showed the existence of hypothesis in $\H_t$ with accuracy $0.5+\Omega(\gamma)$. Standard arguments in learning theory imply that the hypothesis that maximizes the accuracy on a sample also has accuracy $0.5+\Omega(\gamma)$. Specifically, for any sample $S$, denote by $h_S$ the best hypothesis in $\H_t$ on the sample $S$.
Basic arguments in learning theory shows that for a sample of size $m=O\left(\nicefrac{d+\log\frac1\delta}{\gamma^2}\right)$, the hypothesis $h_S$ has a good accuracy, as the following theorem proves.

\begin{restatable}[weak-learner]{thm}{WeakLearnerDis}
\label{thm:weak_learner_dis}
Fix $\alpha>0$. 
For any distribution $\mu$ over  $[-1,+1]^d\times\{-1,+1\}$ that satisfies linear separability with a $\gamma$-margin, and for any $\delta\in(0,1)$ there is $m=O\left(\frac{d+\log\frac1\delta}{\gamma^2}\right)$, such that with probability at least $1-\delta$ over the sample $S$ of size $m$, it holds that $$\Pr_{(x,y)\sim\mu}(h_S(x)=y)\geq \frac12+\frac{\gamma}{4}-\alpha.$$ 
\end{restatable}

\subsection{Decision tree using CART}\label{subsec:id3}
CART is a popular algorithm for learning decision trees. In \cite{kearns1999boosting} it was shown that if the internal nodes define a $\gamma$-weak learner and number of samples is some polynomial of $t\log(1/\delta)d$, then a CART-type algorithm returns a tree with size $t=1/\epsilon^{O(1/\gamma^2)}$ and accuracy at least $1-\epsilon$, with probability at least $1-\delta.$  Under the linearity assumption, we know that the internal nodes indeed define a $\gamma$-weak learner by Theorem~\ref{thm:weak_learner_dis}. Thus, we get a model with a tree size independent of the training size and the dimension. But the model is not necessarily robust.  

The above results can be interrupted as a proof for the CART's algorithm success. This proof does not use the strong assumption of feature independence, which is assumed in recent works \cite{brutzkus2020id3,brutzkus2019optimality,blanc2019top,blanc2020provable,fiat2004decision}.

Designing robust decision trees is inherently a difficult task. 
The reason is that, generally, the model defined by the right and left subtrees can be completely different. The feature $i$ in the root determines if the model uses the right or left subtrees. Thus, a small change in the $i$-th feature completely changes the model. %
To overcome this difficulty we focus on a specific type of decision tree, risk scores \cite{ustun2019learning}, see Table~\ref{tab:risk_score} for an example. In the decision tree that corresponds to the risk score, the right and left subtrees are the same. In the next section, we design risk scores that have guarantees on the robustness and the accuracy.

\subsection{Risk score}\label{subsec:robust}

This section designs an algorithm that returns a risk score model with provable guarantees on its accuracy and robustness, assuming that the data is linearly separable. In the previous section, we used \cite{kearns1999boosting} that viewed CART as a boosting method. This section uses a more traditional boosting method --- the Boost-by-Majority algorithm (BBM) \cite{freund1995boosting}. This boosting algorithm gets as an input training data and an integer $T$, and at each step $t\leq T$ it reweigh the examples and apply a $\gamma$-weak learner that returns a hypothesis $h_t:\R^{d}\rightarrow\{-1,+1\}$. At the end, after $T$ steps, BBM returns $sign\left(\sum_{t=1}^T{h_t}\right).$
In \cite{freund1995boosting,schapire2013boosting} it was shown that BBM returns hypothesis with accuracy at least $1-\epsilon$  after at most $T=O(\gamma^{-2}\log(1/\epsilon))$ rounds. 

The translation from BBM, which uses $\H_t$ as a weak learner, to a risk score model, is straightforward. The hypotheses in $\H_t$ exactly correspond to the conditions in the risk score. Each condition has weight of $1$. If the number of conditions that hold is at least $T/2$ then our risk model returns $+1$, else it returns $-1.$ Together with Theorem~\ref{thm:weak_learner} and \cite{freund1995boosting} we get that BBM returns a risk score with accuracy at least $1-\epsilon$ and with $T=O(\gamma^{-2}\log(1/\epsilon))$ conditions.

We remark that other boosting methods, e.g., \cite{freund1997decision,kanade2009potential}, cannot replace BBM in the suggested scheme, since the final combination has to be a simple sum of the weak learners and \emph{not} arbitrary linear combination. The letter corresponds to a risk score where the weights are in $\R$ and not a small integer, as desired.

Our next and final goal is to prove that our risk score model is also robust. For that, we use the concept of \emph{monotonicity}.
For $x,y\in\R^d$, we say that $x\leq y$ if and only if for all $i\in[d]$ it holds that $x_i\leq y_i$.
A model $f:\R^d\rightarrow\{0,1\}$ is monotone if for all $x\leq y$ it holds that $f(x)\leq f(y).$ 
We will show that BBM with weak learners from $\H_t$ yields a monotone model.
The reasons being (i) all conditions are of the form ``$x_i\geq \theta$", (ii) all weights are non-negative, except the bias term, and (iii) classification of a risk score is detriment by the score's sign. 
All proofs appear in Section~\ref{apx:risk_score}.

\begin{restatable}{clm}{RsMonotone}
\label{clm:our_risk_monotone}
If every condition in a risk-score model $R$ is of the form ``$x_i\geq \theta$" and all weights are positive, except the bias term, then $R$ is a monotone model. 
\end{restatable}

In Claim~\ref{clm:monotone_implies_robust} we show that, by carefully adding a small noise to each feature, we can transform any algorithm that returns a monotone model to one that returns a robust model.

\begin{restatable}{clm}{MonotoneImpliesRobust}
\label{clm:monotone_implies_robust}
Assume a learning algorithm $A$ gets as an input $\gamma$-linearly separable and returns a \emph{monotone} model with accuracy $1-\epsilon(\gamma)$. Then, there is an algorithm that returns a model with astuteness at least $1-\epsilon\left(\frac\gamma 2\right)$ at radius $\gamma/2.$
\end{restatable}

To summarize, in Algorithm~\ref{alg:BBM_RS} we show the pseudocode of our new algorithm, BBM-RS. In the first step we add noise to each example by replacing each example $(x,y)$ by $(x-\tau y\mathbf{1},y)$, where $\tau\in(0,1)$ is a parameter that defines the noise level and $\mathbf{1}$ is the all-one vector. 
In other words, we add noise $y\tau$ to each feature.
In the second step, the algorithm iteratively adds conditions to the risk score. At each iteration, we first find the distribution $\mu$ defined by BBM \cite{freund1995boosting}. Then, we find the best hypothesis $h_{i,\theta}$ in $\H_t$, according to $\mu.$ We add to the risk score a condition  ``$x_i\geq\theta$". Finally, we add a bias term of $-T/2$, to check if at least half of the conditions are satisfied. 
\begin{algorithm} 
\caption{BBM-RS (BBM-Risk Score)}\label{alg:BBM_RS}
\begin{algorithmic}
\STATE \textbf{input:} $D$: linearly separable training data by $w$;\quad WLOG $\forall i. w_i\geq 0$
\STATE \quad\quad\quad $T$: bound on interpretation complexity 
\STATE \quad\quad\quad $\tau$: noise level 
\STATE \textbf{output:} risk score
\STATE \# Add noise:
\FOR {$(x,y)\in D$}
\STATE replace $(x,y)$ with  $(x-\tau y \mathbf{1},y)$
\ENDFOR
\FOR {$i=1\ldots T$}
\STATE $\mu\leftarrow$ BBM distrbution on $D$
\STATE $i,\theta \leftarrow \argmax_{i,\theta} \sum_{(x,y)\in D}\mu(x)I_{(x_i-\theta)y>0}$
\STATE Add condition ``$x_i\geq\theta$" to $RS$
\ENDFOR
\STATE Add a bias term of $-T/2$ to $RS$
\STATE \textbf{return} $RS$
\end{algorithmic}
\end{algorithm}

\section{Experiments}\label{sec:experiment}
In previous sections, we designed new algorithms 
and gave provable guarantees for separated data.
We next investigate these results on real datasets.
Concretely, we ask the following questions:
\vspace{-1em}
\begin{itemize}
    \setlength\itemsep{.05em}
    \item How separated are real datasets?
    \item How well does BBM-RS perform compared with other interpretable methods?
    \item How do interpretability, robustness, and accuracy trade-off with one another in BBM-RS?
\end{itemize}

\textbf{Datasets.}
To maintain compatibility with prior work on interpretable and robust decision trees~\cite{ustun2019learning,lin2020generalized}, we use the following pre-processed datasets
from their repositories -- adult, bank, breastcancer, mammo, mushroom,
spambase, careval, ficobin, and campasbin.
We also use some datasets from other sources such as 
LIBSVM~\cite{chang2011libsvm} datasets and \citet{moro2014data}.
These include diabetes, heart, ionosphere, and bank2.
All features are normalized to $[0, 1]$.
More details can be found in Appendix~\ref{app:experiment}.
The dataset statistics are shown in Table~\ref{tab:dataset_stats}.

\begin{table}[ht]
    \setlength{\tabcolsep}{2pt}
    \scriptsize
    \centering
\begin{tabular}{lcccc|cc|cc}
\toprule
{} & \multicolumn{4}{c}{dataset statistics} & \multicolumn{2}{c}{\thead{$r$-separation}} & \multicolumn{2}{c}{\thead{$\gamma$-linear \\ separation}} \\
{} &  \thead{\# \\ samples} &  \thead{\# \\ features} &  \thead{\# \\ binary \\ features} &  \thead{portion of \\ positive\\ label} & \thead{sep.} & \thead{$2r$} &  \thead{sep.} & \thead{$\gamma$}\\
\midrule
adult        &       32561 &           36 &                  36 &                          0.24 &                0.88 &  1.00 &                    0.84 & 0.001 \\
bank         &       41188 &           57 &                  57 &                          0.11 &                0.97 &  1.00 &                    0.90 & 0.33 \\
bank2        &       41188 &           63 &                  53 &                          0.11 &                1.00 &  0.0004 &                    0.91 & 0.00002 \\
breastcancer &         683 &            9 &                   0 &                          0.35 &                1.00 &  0.11 &                    0.97 & 0.0003 \\
careval      &        1728 &           15 &                  15 &                          0.30 &                1.00 &  1.00 &                    0.96 & 0.003 \\
compasbin    &        6907 &           12 &                  12 &                          0.46 &                0.68 &  1.00 &                    0.65 & 0.20 \\
diabetes     &         768 &            8 &                   0 &                          0.65 &                1.00 &  0.11 &                    0.77 & 0.0008 \\
ficobin      &       10459 &           17 &                  17 &                          0.48 &                0.79 &  1.00 &                    0.70 & 0.33 \\
heart        &         270 &           20 &                  13 &                          0.44 &                1.00 &  0.13 &                    0.89 & 0.0003 \\
ionosphere   &         351 &           34 &                   1 &                          0.64 &                1.00 &  0.80 &                    0.95 & 0.0007 \\
mammo        &         961 &           14 &                  13 &                          0.46 &                0.83 &  0.33 &                    0.79 & 0.14 \\
mushroom     &        8124 &          113 &                 113 &                          0.48 &                1.00 &  1.00 &                    1.00 & 0.02 \\
spambase     &        4601 &           57 &                   0 &                          0.39 &                1.00 &  0.000063 &                    0.94 & 0.000002 \\
\bottomrule
\end{tabular}
     \vspace{-0.5em}
    \caption{Dataset statistics.
    Columns ``sep.'' records the separateness of each dataset.
    Columns ``$2r$" and ``$\gamma$" are calculated after dataset is separated by removing $1-\text{sep}$ points. %
    }
    \label{tab:dataset_stats}
\end{table}

\subsection{Separation of real datasets}
To understand how separated they are, we measure the closeness of each dataset to being $r$- or linearly separated.
The \textit{separateness} of a dataset is one minus the fraction of examples needed to be removed
for it to be $r$- or linearly separated. %

For $r$-separation, we use the algorithm designed by \citet{yang2020robustness} that calculates the minimum number of examples needed to be removed for a dataset to be $r$-separated with $r \geq 10^{-5}$.
This ensures that after removal, there will be no pair of examples 
that are very similar but with different labels. Finding the optimal separateness for linear separation is NP-hard \cite{ben2003difficulty}, thus %
we run a $\ell_1$ regularized linear SVM with
regularization terms $C = \{10^{-10}, 10^{-8}, \ldots, 10^{10}\}$ and record the lowest training error as an approximation to one minus the optimal separateness.  %

The separation results are shown in Table~\ref{tab:dataset_stats}.
Eight datasets are already $r$-separated (separateness $=100\%$). In the five datasets with separateness $<$ 100\%, there are examples with very similar features but different labels. This occurs mostly in binarized datasets;  see Appendix~\ref{apx:additional_results} for an example. Three datasets are almost separated with separateness equal to $97\%$, $88\%$, and $83\%$, and two have separateness $68\%$ and $79\%$. To summarize, $84\%$ of the datasets are $r$-separated with $r\geq 10^{-5}$, after removing at most $17\%$ of the points.

Linear separation is a stricter property than $r$-separation, so the separateness for linear separation is smaller or equal to the separateness for $r$-separation.
Seven datasets have separateness $\geq 90\%$, three separateness between $79\%$ and $89\%$, and the remaining three have separateness $< 79\%$. 
After removing the points, all datasets are $\gamma$-linearly separable and nine datasets have  $\gamma\geq 0.001.$
To summarize, (i) $77\%$ of the datasets are close to being linearly separated (ii) requiring linear-separability reduces the separateness of the $r$-separated dataset by only an average of $6.77\%$.
From this we conclude that for these datasets at least, the assumption of $r-$ or 
linear-separability is approximately correct. %

\begin{table*}[ht]
    \footnotesize
    \setlength{\tabcolsep}{4pt}
    \centering
\begin{tabular}{lcccc|cccc|cccc}
\toprule
{} & \multicolumn{4}{c}{IC (lower=better)} & \multicolumn{4}{c}{test accuracy (higher=better)} & \multicolumn{4}{c}{ER (higher=better)} \\
{} &                DT &             RobDT &            LCPA &            BBM-RS &               DT &            RobDT &            LCPA &           BBM-RS &               DT &            RobDT &            LCPA &           BBM-RS \\
\midrule
adult        &  $414.20$ &         $287.90$ &          $14.90$ &  $\mathbf{6.00}$ &  $\mathbf{0.83}$ &  $\mathbf{0.83}$ &           $0.82$ &           $0.81$ &  $\mathbf{0.50}$ &  $\mathbf{0.50}$ &   $0.12$ &  $\mathbf{0.50}$ \\
bank         &   $30.70$ &          $26.80$ &           $8.90$ &  $\mathbf{8.00}$ &  $\mathbf{0.90}$ &  $\mathbf{0.90}$ &  $\mathbf{0.90}$ &  $\mathbf{0.90}$ &  $\mathbf{0.50}$ &  $\mathbf{0.50}$ &   $0.20$ &  $\mathbf{0.50}$ \\
bank2        &   $30.00$ &          $30.70$ &          $13.80$ &  $\mathbf{4.50}$ &  $\mathbf{0.91}$ &           $0.90$ &           $0.90$ &           $0.90$ &           $0.12$ &           $0.18$ &   $0.10$ &  $\mathbf{0.50}$ \\
breastcancer &   $15.20$ &           $7.40$ &  $\mathbf{6.00}$ &          $11.00$ &           $0.94$ &           $0.94$ &  $\mathbf{0.96}$ &  $\mathbf{0.96}$ &           $0.23$ &  $\mathbf{0.29}$ &   $0.28$ &           $0.27$ \\
careval      &   $59.30$ &          $28.20$ &          $10.10$ &  $\mathbf{8.70}$ &  $\mathbf{0.97}$ &           $0.96$ &           $0.91$ &           $0.77$ &  $\mathbf{0.50}$ &  $\mathbf{0.50}$ &   $0.19$ &  $\mathbf{0.50}$ \\
compasbin    &   $67.80$ &          $33.70$ &  $\mathbf{5.40}$ &           $7.60$ &  $\mathbf{0.67}$ &  $\mathbf{0.67}$ &           $0.65$ &           $0.66$ &  $\mathbf{0.50}$ &  $\mathbf{0.50}$ &   $0.15$ &           $0.33$ \\
diabetes     &   $31.20$ &          $27.90$ &           $6.00$ &  $\mathbf{2.10}$ &           $0.74$ &           $0.73$ &  $\mathbf{0.76}$ &           $0.65$ &           $0.08$ &           $0.08$ &   $0.09$ &  $\mathbf{0.15}$ \\
ficobin      &   $30.60$ &          $59.60$ &  $\mathbf{6.40}$ &          $11.80$ &           $0.71$ &           $0.71$ &           $0.71$ &  $\mathbf{0.72}$ &  $\mathbf{0.50}$ &  $\mathbf{0.50}$ &   $0.22$ &  $\mathbf{0.50}$ \\
heart        &   $20.30$ &          $13.60$ &          $11.90$ &  $\mathbf{9.50}$ &           $0.76$ &           $0.79$ &  $\mathbf{0.82}$ &  $\mathbf{0.82}$ &           $0.23$ &           $0.31$ &   $0.14$ &  $\mathbf{0.32}$ \\
ionosphere   &   $11.30$ &           $8.60$ &          $17.90$ &  $\mathbf{6.80}$ &           $0.89$ &  $\mathbf{0.92}$ &           $0.88$ &           $0.86$ &           $0.15$ &           $0.25$ &   $0.07$ &  $\mathbf{0.28}$ \\
mammo        &   $27.40$ &          $12.40$ &           $7.20$ &  $\mathbf{1.90}$ &  $\mathbf{0.79}$ &  $\mathbf{0.79}$ &  $\mathbf{0.79}$ &           $0.77$ &           $0.47$ &  $\mathbf{0.50}$ &   $0.21$ &  $\mathbf{0.50}$ \\
mushroom     &   $10.80$ &  $\mathbf{9.10}$ &          $23.80$ &           $9.90$ &  $\mathbf{1.00}$ &  $\mathbf{1.00}$ &  $\mathbf{1.00}$ &           $0.97$ &  $\mathbf{0.50}$ &  $\mathbf{0.50}$ &   $0.10$ &  $\mathbf{0.50}$ \\
spambase     &  $153.90$ &          $72.30$ &          $29.50$ &  $\mathbf{5.60}$ &  $\mathbf{0.92}$ &           $0.87$ &           $0.88$ &           $0.79$ &           $0.00$ &           $0.04$ &   $0.02$ &  $\mathbf{0.05}$ \\
\bottomrule
\end{tabular}
     \vspace{-.5em}
    \caption{Comparison of BBM-RS with other interpretable models. In bold: the best algorithm for each dataset and criterion.
    Note that several datasets (adult, bank, careval, compasbin, ficobin, and mushroom)
    have ER = $0.5$ for tree-based models (DT, RobDT, and BBM-RS), because these datasets have all binary features and tree-based models set the threshold in the middle of $0$ and $1$.
    }%
    \label{tab:bbm_cmp}
\end{table*}

\subsection{Performance of BBM-RS}
Next, we want to understand how our proposed BBM-RS performs on real datasets.
We compare the performance of BBM-RS with three different baselines
on three evaluation criteria: interpretability, accuracy, and robustness.

\textbf{Baselines.}
We compare BBM-RS with three baselines:
(i) LCPA~\cite{ustun2019learning}, an algorithm for learning risk scores,
(ii) DT~\cite{breiman1984classification}, standard algorithm for learning decision trees,
and (iii) Robust decision tree~(RobDT)~\cite{chen2019robust}, an algorithm for learning robust decision trees.

We use a $5$-fold cross-validation based on accuracy for hyperparameters selection.
For DT and RobDT, we search through $5, 10, \ldots 30$ for the maximum depth of the tree.
For BBM-RS, we search through $5, 10, \ldots 30$ for the maximum number of weak learners ($T$). 
The algorithm stops when it reaches $T$ iterations or if no weak learner can
produce a weighted accuracy $> 0.51$. 
For LCPA, we search through $5, 10, \ldots 30$ for the maximum $\ell_0$ norm of the weight vector.
We set the robust radius for RobDT and the noise level $\tau$ for BBM-RS to $0.05$.
More details about the setup of the algorithms can be found in Appendix~\ref{app:experiment}.

\subsubsection{Evaluation}

We evaluate interpretability, accuracy, and robustness of each baseline.
The data is randomly split into training and testing sets by 2:1.
The experiment is repeated $10$ times with different training and testing splits.
The mean and standard error of the evaluation criteria are recorded.

\textbf{Interpretability.}
We measure a model's interpretability by evaluating its \textit{Interpretation Complexity~(IC)}, which is the number of parameters in the model.
For decision trees (DT and RobDT) the IC is the number of internal nodes in the tree,
and for risk scores (LCPA and BBM-RS) it is the number non-zero terms in the weight vector.
The lower the IC is, the more interpretable the model is.

\textbf{Robustness.}
We measure model's robustness by evaluating its
\textit{Empirical robustness~(ER)}~\cite{yang2020robustness}.
ER on a classifier $f$ at an input $x$
is $ER(f, x) := \min_{f(x') \neq f(x)} \|x' - x\|_\infty$. 
We evaluate ER on $100$ randomly chosen
correctly predicted examples in the test set.
The larger ER is, the more robust the classifier is.

\subsubsection{Results}

The results are shown in Table~\ref{tab:bbm_cmp}
(only the means are shown, the standard errors can be found in Appendix~\ref{apx:additional_results}).
We see that BBM-RS performs well in terms of interpretability
and robustness.
BBM-RS performs the best on nine and eleven out of thirteen datasets in terms of interpretation complexity and robustness, respectively.
In terms of accuracy, in nine out of the thirteen datasets, BBM-RS is the best or within $3\%$ to the best. 
These results show that on most datasets, BBM-RS is better than other algorithms in IC and ER while being comparable in accuracy. 

\subsection{Tradeoffs in BBM-RS}
The parameter $\tau$ gives us the opportunity to explore the tradeoff between interpretability, robustness, and accuracy within BBM-RS.
Figure~\ref{fig:tau_tradeoff} shows that for small $\tau$, BBM-RS's IC is high, and its ER is low, and when $\tau$ is high, IC is low, and ER is high. %
This empirical observation strengthens the claim %
that interpretability and robustness are correlated.
See Appendix~\ref{apx:additional_results} for experiments on other datasets and experiments on the tradeoffs between IC and accuracy.

\begin{figure}[!h]
    \centering
    \includegraphics[width=.24\textwidth]{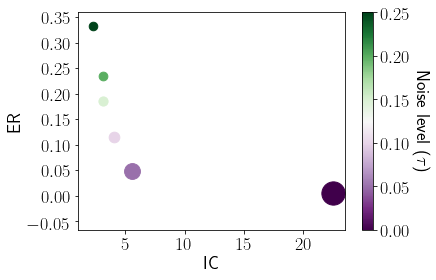}
    \vspace{-.8em}
    \caption{Interaction of interpretability, accuracy, and robustness with different noise level $\tau$ on the spambase dataset.
    The size of each ball represents the accuracy. For $\tau=0$: $IC=22.5, ER=0.006$ and for higher noise $\tau=0.25$: $IC=2.3,ER=0.33$}
    \label{fig:tau_tradeoff}
\end{figure}
\vspace{-1em}

\section{Conclusion}
We found that linear separability is a hidden property of the data that guarantees both interpretability and robustness. We designed an efficient algorithm, BBM-RS, that returns a model, risk-score, which we prove is  interpretable, robust, and have high-accuracy. An interesting open question is whether a weaker notion than linear separability can give similar guarantees.

\section*{Acknowledgements}
Kamalika Chaudhuri and Yao-Yuan Yang thank NSF under CIF 1719133 and CNS 1804829 for research support.

\bibliographystyle{icml2021}
\bibliography{main}

\begin{thebibliography}{61}
\providecommand{\natexlab}[1]{#1}
\providecommand{\url}[1]{\texttt{#1}}
\expandafter\ifx\csname urlstyle\endcsname\relax
  \providecommand{\doi}[1]{doi: #1}\else
  \providecommand{\doi}{doi: \begingroup \urlstyle{rm}\Url}\fi

\bibitem[Adebayo et~al.(2018)Adebayo, Gilmer, Muelly, Goodfellow, Hardt, and
  Kim]{adebayo2018sanity}
Adebayo, J., Gilmer, J., Muelly, M., Goodfellow, I., Hardt, M., and Kim, B.
\newblock Sanity checks for saliency maps.
\newblock \emph{Advances in neural information processing systems},
  31:\penalty0 9505--9515, 2018.

\bibitem[Alon \& Spencer(2004)Alon and Spencer]{alon2004probabilistic}
Alon, N. and Spencer, J.~H.
\newblock \emph{The probabilistic method}.
\newblock John Wiley \& Sons, 2004.

\bibitem[Andriushchenko \& Hein(2019)Andriushchenko and
  Hein]{andriushchenko2019provably}
Andriushchenko, M. and Hein, M.
\newblock Provably robust boosted decision stumps and trees against adversarial
  attacks.
\newblock In \emph{Advances in Neural Information Processing Systems}, pp.\
  13017--13028, 2019.

\bibitem[Ben-David et~al.(2003)Ben-David, Eiron, and Long]{ben2003difficulty}
Ben-David, S., Eiron, N., and Long, P.~M.
\newblock On the difficulty of approximately maximizing agreements.
\newblock \emph{Journal of Computer and System Sciences}, 66\penalty0
  (3):\penalty0 496--514, 2003.

\bibitem[Blanc et~al.(2019)Blanc, Lange, and Tan]{blanc2019top}
Blanc, G., Lange, J., and Tan, L.-Y.
\newblock Top-down induction of decision trees: rigorous guarantees and
  inherent limitations.
\newblock \emph{arXiv preprint arXiv:1911.07375}, 2019.

\bibitem[Blanc et~al.(2020)Blanc, Lange, and Tan]{blanc2020provable}
Blanc, G., Lange, J., and Tan, L.-Y.
\newblock Provable guarantees for decision tree induction: the agnostic
  setting.
\newblock \emph{arXiv preprint arXiv:2006.00743}, 2020.

\bibitem[Boer et~al.(2020)Boer, Deutch, Frost, and Milo]{boer2020personal}
Boer, N., Deutch, D., Frost, N., and Milo, T.
\newblock Personal insights for altering decisions of tree-based ensembles over
  time.
\newblock \emph{Proceedings of the VLDB Endowment}, 13\penalty0 (6):\penalty0
  798--811, 2020.

\bibitem[Breiman et~al.(1984)Breiman, Friedman, Stone, and
  Olshen]{breiman1984classification}
Breiman, L., Friedman, J., Stone, C.~J., and Olshen, R.~A.
\newblock \emph{Classification and regression trees}.
\newblock CRC press, 1984.

\bibitem[Brutzkus et~al.(2019)Brutzkus, Daniely, and
  Malach]{brutzkus2019optimality}
Brutzkus, A., Daniely, A., and Malach, E.
\newblock On the optimality of trees generated by id3.
\newblock \emph{arXiv preprint arXiv:1907.05444}, 2019.

\bibitem[Brutzkus et~al.(2020)Brutzkus, Daniely, and Malach]{brutzkus2020id3}
Brutzkus, A., Daniely, A., and Malach, E.
\newblock Id3 learns juntas for smoothed product distributions.
\newblock In \emph{Conference on Learning Theory}, pp.\  902--915. PMLR, 2020.

\bibitem[Chang \& Lin(2011)Chang and Lin]{chang2011libsvm}
Chang, C.-C. and Lin, C.-J.
\newblock Libsvm: a library for support vector machines.
\newblock \emph{ACM transactions on intelligent systems and technology (TIST)},
  2\penalty0 (3):\penalty0 1--27, 2011.

\bibitem[Chen et~al.(2019)Chen, Zhang, Boning, and Hsieh]{chen2019robust}
Chen, H., Zhang, H., Boning, D., and Hsieh, C.-J.
\newblock Robust decision trees against adversarial examples.
\newblock \emph{arXiv preprint arXiv:1902.10660}, 2019.

\bibitem[Cohen et~al.(2019)Cohen, Rosenfeld, and Kolter]{cohen2019certified}
Cohen, J.~M., Rosenfeld, E., and Kolter, J.~Z.
\newblock Certified adversarial robustness via randomized smoothing.
\newblock \emph{arXiv preprint arXiv:1902.02918}, 2019.

\bibitem[Deutch \& Frost(2019)Deutch and Frost]{deutch2019constraints}
Deutch, D. and Frost, N.
\newblock Constraints-based explanations of classifications.
\newblock In \emph{2019 IEEE 35th International Conference on Data Engineering
  (ICDE)}, pp.\  530--541. IEEE, 2019.

\bibitem[Fiat \& Pechyony(2004)Fiat and Pechyony]{fiat2004decision}
Fiat, A. and Pechyony, D.
\newblock Decision trees: More theoretical justification for practical
  algorithms.
\newblock In \emph{International Conference on Algorithmic Learning Theory},
  pp.\  156--170. Springer, 2004.

\bibitem[Freund(1995)]{freund1995boosting}
Freund, Y.
\newblock Boosting a weak learning algorithm by majority.
\newblock \emph{Information and computation}, 121\penalty0 (2):\penalty0
  256--285, 1995.

\bibitem[Freund \& Schapire(1997)Freund and Schapire]{freund1997decision}
Freund, Y. and Schapire, R.~E.
\newblock A decision-theoretic generalization of on-line learning and an
  application to boosting.
\newblock \emph{Journal of computer and system sciences}, 55\penalty0
  (1):\penalty0 119--139, 1997.

\bibitem[Frost et~al.(2020)Frost, Moshkovitz, and Rashtchian]{frost2020exkmc}
Frost, N., Moshkovitz, M., and Rashtchian, C.
\newblock Exkmc: Expanding explainable $ k $-means clustering.
\newblock \emph{arXiv preprint arXiv:2006.02399}, 2020.

\bibitem[Garreau \& von Luxburg(2020{\natexlab{a}})Garreau and von
  Luxburg]{garreau2020explaining}
Garreau, D. and von Luxburg, U.
\newblock Explaining the explainer: A first theoretical analysis of lime.
\newblock \emph{arXiv preprint arXiv:2001.03447}, 2020{\natexlab{a}}.

\bibitem[Garreau \& von Luxburg(2020{\natexlab{b}})Garreau and von
  Luxburg]{garreau2020looking}
Garreau, D. and von Luxburg, U.
\newblock Looking deeper into lime.
\newblock \emph{arXiv preprint arXiv:2008.11092}, 2020{\natexlab{b}}.

\bibitem[Goodfellow et~al.(2014)Goodfellow, Shlens, and
  Szegedy]{goodfellow2014explaining}
Goodfellow, I.~J., Shlens, J., and Szegedy, C.
\newblock Explaining and harnessing adversarial examples.
\newblock \emph{arXiv preprint arXiv:1412.6572}, 2014.

\bibitem[Hu et~al.(2019)Hu, Rudin, and Seltzer]{hu2019optimal}
Hu, X., Rudin, C., and Seltzer, M.
\newblock Optimal sparse decision trees.
\newblock In \emph{Advances in Neural Information Processing Systems}, pp.\
  7267--7275, 2019.

\bibitem[Ignatiev et~al.(2019)Ignatiev, Narodytska, and
  Marques-Silva]{ignatiev2019relating}
Ignatiev, A., Narodytska, N., and Marques-Silva, J.
\newblock On relating explanations and adversarial examples.
\newblock In \emph{Advances in Neural Information Processing Systems}, pp.\
  15883--15893, 2019.

\bibitem[Kanade \& Kalai(2009)Kanade and Kalai]{kanade2009potential}
Kanade, V. and Kalai, A.
\newblock Potential-based agnostic boosting.
\newblock \emph{Advances in neural information processing systems},
  22:\penalty0 880--888, 2009.

\bibitem[Kantchelian et~al.(2016)Kantchelian, Tygar, and
  Joseph]{kantchelian2016evasion}
Kantchelian, A., Tygar, J.~D., and Joseph, A.
\newblock Evasion and hardening of tree ensemble classifiers.
\newblock In \emph{International Conference on Machine Learning}, pp.\
  2387--2396, 2016.

\bibitem[Kearns(1988)]{kearns1988learning}
Kearns, M.
\newblock Learning boolean formulae or finite automata is as hard as factoring.
\newblock \emph{Technical Report TR-14-88 Harvard University Aikem Computation
  Laboratory}, 1988.

\bibitem[Kearns \& Mansour(1999)Kearns and Mansour]{kearns1999boosting}
Kearns, M. and Mansour, Y.
\newblock On the boosting ability of top--down decision tree learning
  algorithms.
\newblock \emph{Journal of Computer and System Sciences}, 58\penalty0
  (1):\penalty0 109--128, 1999.

\bibitem[Kindermans et~al.(2019)Kindermans, Hooker, Adebayo, Alber, Sch{\"u}tt,
  D{\"a}hne, Erhan, and Kim]{kindermans2019reliability}
Kindermans, P.-J., Hooker, S., Adebayo, J., Alber, M., Sch{\"u}tt, K.~T.,
  D{\"a}hne, S., Erhan, D., and Kim, B.
\newblock The (un) reliability of saliency methods.
\newblock In \emph{Explainable AI: Interpreting, Explaining and Visualizing
  Deep Learning}, pp.\  267--280. Springer, 2019.

\bibitem[Koh \& Liang(2017)Koh and Liang]{koh2017understanding}
Koh, P.~W. and Liang, P.
\newblock Understanding black-box predictions via influence functions.
\newblock \emph{arXiv preprint arXiv:1703.04730}, 2017.

\bibitem[Laber \& Murtinho(2021)Laber and Murtinho]{laber2021price}
Laber, E. and Murtinho, L.
\newblock On the price of explainability for some clustering problems.
\newblock \emph{arXiv preprint arXiv:2101.01576}, 2021.

\bibitem[Lakkaraju et~al.(2020)Lakkaraju, Arsov, and
  Bastani]{lakkaraju2020robust}
Lakkaraju, H., Arsov, N., and Bastani, O.
\newblock Robust and stable black box explanations.
\newblock In \emph{International Conference on Machine Learning}, pp.\
  5628--5638. PMLR, 2020.

\bibitem[Li et~al.(2020)Li, Shi, Li, Bai, Song, Cao, and
  Chen]{li2020quantitative}
Li, X.-H., Shi, Y., Li, H., Bai, W., Song, Y., Cao, C.~C., and Chen, L.
\newblock Quantitative evaluations on saliency methods: An experimental study.
\newblock \emph{arXiv preprint arXiv:2012.15616}, 2020.

\bibitem[Lin et~al.(2020)Lin, Zhong, Hu, Rudin, and
  Seltzer]{lin2020generalized}
Lin, J., Zhong, C., Hu, D., Rudin, C., and Seltzer, M.
\newblock Generalized and scalable optimal sparse decision trees.
\newblock In \emph{International Conference on Machine Learning}, pp.\
  6150--6160. PMLR, 2020.

\bibitem[Lundberg \& Lee(2017)Lundberg and Lee]{lundberg2017unified}
Lundberg, S.~M. and Lee, S.-I.
\newblock A unified approach to interpreting model predictions.
\newblock In \emph{Advances in neural information processing systems}, pp.\
  4765--4774, 2017.

\bibitem[Lundberg et~al.(2018)Lundberg, Erion, and Lee]{lundberg2018consistent}
Lundberg, S.~M., Erion, G.~G., and Lee, S.-I.
\newblock Consistent individualized feature attribution for tree ensembles.
\newblock \emph{arXiv preprint arXiv:1802.03888}, 2018.

\bibitem[Madry et~al.(2017)Madry, Makelov, Schmidt, Tsipras, and
  Vladu]{madry2017towards}
Madry, A., Makelov, A., Schmidt, L., Tsipras, D., and Vladu, A.
\newblock Towards deep learning models resistant to adversarial attacks.
\newblock \emph{arXiv preprint arXiv:1706.06083}, 2017.

\bibitem[Mardaoui \& Garreau(2020)Mardaoui and Garreau]{mardaoui2020analysis}
Mardaoui, D. and Garreau, D.
\newblock An analysis of lime for text data.
\newblock \emph{arXiv preprint arXiv:2010.12487}, 2020.

\bibitem[Molnar(2019)]{molnar2020interpretable}
Molnar, C.
\newblock \emph{Interpretable Machine Learning}.
\newblock 2019.
\newblock \url{https://christophm.github.io/interpretable-ml-book/}.

\bibitem[Moro et~al.(2014)Moro, Cortez, and Rita]{moro2014data}
Moro, S., Cortez, P., and Rita, P.
\newblock A data-driven approach to predict the success of bank telemarketing.
\newblock \emph{Decision Support Systems}, 62:\penalty0 22--31, 2014.

\bibitem[Moshkovitz et~al.(2020)Moshkovitz, Dasgupta, Rashtchian, and
  Frost]{moshkovitz2020explainable}
Moshkovitz, M., Dasgupta, S., Rashtchian, C., and Frost, N.
\newblock Explainable k-means and k-medians clustering.
\newblock In \emph{International Conference on Machine Learning}, pp.\
  7055--7065. PMLR, 2020.

\bibitem[Nacson et~al.(2019)Nacson, Lee, Gunasekar, Savarese, Srebro, and
  Soudry]{nacson2019convergence}
Nacson, M.~S., Lee, J., Gunasekar, S., Savarese, P. H.~P., Srebro, N., and
  Soudry, D.
\newblock Convergence of gradient descent on separable data.
\newblock In \emph{The 22nd International Conference on Artificial Intelligence
  and Statistics}, pp.\  3420--3428. PMLR, 2019.

\bibitem[Pedregosa et~al.(2011)Pedregosa, Varoquaux, Gramfort, Michel, Thirion,
  Grisel, Blondel, Prettenhofer, Weiss, Dubourg, Vanderplas, Passos,
  Cournapeau, Brucher, Perrot, and Duchesnay]{scikit-learn}
Pedregosa, F., Varoquaux, G., Gramfort, A., Michel, V., Thirion, B., Grisel,
  O., Blondel, M., Prettenhofer, P., Weiss, R., Dubourg, V., Vanderplas, J.,
  Passos, A., Cournapeau, D., Brucher, M., Perrot, M., and Duchesnay, E.
\newblock Scikit-learn: Machine learning in {P}ython.
\newblock \emph{Journal of Machine Learning Research}, 12:\penalty0 2825--2830,
  2011.

\bibitem[Quinlan(1986)]{quinlan1986induction}
Quinlan, J.~R.
\newblock Induction of decision trees.
\newblock \emph{Machine learning}, 1\penalty0 (1):\penalty0 81--106, 1986.

\bibitem[Ribeiro et~al.(2016{\natexlab{a}})Ribeiro, Singh, and
  Guestrin]{ribeiro2016model}
Ribeiro, M.~T., Singh, S., and Guestrin, C.
\newblock Model-agnostic interpretability of machine learning.
\newblock \emph{arXiv preprint arXiv:1606.05386}, 2016{\natexlab{a}}.

\bibitem[Ribeiro et~al.(2016{\natexlab{b}})Ribeiro, Singh, and
  Guestrin]{ribeiro2016should}
Ribeiro, M.~T., Singh, S., and Guestrin, C.
\newblock " why should i trust you?" explaining the predictions of any
  classifier.
\newblock In \emph{Proceedings of the 22nd ACM SIGKDD international conference
  on knowledge discovery and data mining}, pp.\  1135--1144,
  2016{\natexlab{b}}.

\bibitem[Ribeiro et~al.(2018)Ribeiro, Singh, and Guestrin]{ribeiro2018anchors}
Ribeiro, M.~T., Singh, S., and Guestrin, C.
\newblock Anchors: High-precision model-agnostic explanations.
\newblock In \emph{Proceedings of the AAAI Conference on Artificial
  Intelligence}, volume~32, 2018.

\bibitem[Ross \& Doshi-Velez(2017)Ross and Doshi-Velez]{ross2017improving}
Ross, A.~S. and Doshi-Velez, F.
\newblock Improving the adversarial robustness and interpretability of deep
  neural networks by regularizing their input gradients.
\newblock \emph{arXiv preprint arXiv:1711.09404}, 2017.

\bibitem[Rudin(2019{\natexlab{a}})]{Rudin19}
Rudin, C.
\newblock Stop explaining black box machine learning models for high stakes
  decisions and use interpretable models instead.
\newblock \emph{Nature Machine Intelligence}, 1:\penalty0 206--215, May
  2019{\natexlab{a}}.

\bibitem[Rudin(2019{\natexlab{b}})]{rudin2019stop}
Rudin, C.
\newblock Stop explaining black box machine learning models for high stakes
  decisions and use interpretable models instead.
\newblock \emph{Nature Machine Intelligence}, 1\penalty0 (5):\penalty0
  206--215, 2019{\natexlab{b}}.

\bibitem[Schapire \& Freund(2013)Schapire and Freund]{schapire2013boosting}
Schapire, R.~E. and Freund, Y.
\newblock Boosting: Foundations and algorithms.
\newblock \emph{Kybernetes}, 2013.

\bibitem[Shalev-Shwartz \& Ben-David(2014)Shalev-Shwartz and
  Ben-David]{shalev2014understanding}
Shalev-Shwartz, S. and Ben-David, S.
\newblock \emph{Understanding machine learning: From theory to algorithms}.
\newblock Cambridge university press, 2014.

\bibitem[Shalev-Shwartz \& Singer(2008)Shalev-Shwartz and
  Singer]{shalev2008equivalence}
Shalev-Shwartz, S. and Singer, Y.
\newblock On the equivalence of weak learnability and linear separability: New
  relaxations and efficient boosting algorithms.
\newblock In \emph{21st Annual Conference on Learning Theory, COLT 2008}, 2008.

\bibitem[Shamir(2020)]{shamir2020gradient}
Shamir, O.
\newblock Gradient methods never overfit on separable data.
\newblock \emph{arXiv preprint arXiv:2007.00028}, 2020.

\bibitem[Slack et~al.(2020)Slack, Hilgard, Jia, Singh, and
  Lakkaraju]{slack2020fooling}
Slack, D., Hilgard, S., Jia, E., Singh, S., and Lakkaraju, H.
\newblock Fooling lime and shap: Adversarial attacks on post hoc explanation
  methods.
\newblock In \emph{Proceedings of the AAAI/ACM Conference on AI, Ethics, and
  Society}, pp.\  180--186, 2020.

\bibitem[Soudry et~al.(2018)Soudry, Hoffer, Nacson, Gunasekar, and
  Srebro]{soudry2018implicit}
Soudry, D., Hoffer, E., Nacson, M.~S., Gunasekar, S., and Srebro, N.
\newblock The implicit bias of gradient descent on separable data.
\newblock \emph{The Journal of Machine Learning Research}, 19\penalty0
  (1):\penalty0 2822--2878, 2018.

\bibitem[Szegedy et~al.(2013)Szegedy, Zaremba, Sutskever, Bruna, Erhan,
  Goodfellow, and Fergus]{szegedy2013intriguing}
Szegedy, C., Zaremba, W., Sutskever, I., Bruna, J., Erhan, D., Goodfellow, I.,
  and Fergus, R.
\newblock Intriguing properties of neural networks.
\newblock \emph{arXiv preprint arXiv:1312.6199}, 2013.

\bibitem[Ustun \& Rudin(2017)Ustun and Rudin]{ustun2017optimized}
Ustun, B. and Rudin, C.
\newblock Optimized risk scores.
\newblock In \emph{Proceedings of the 23rd ACM SIGKDD international conference
  on knowledge discovery and data mining}, pp.\  1125--1134, 2017.

\bibitem[Ustun \& Rudin(2019)Ustun and Rudin]{ustun2019learning}
Ustun, B. and Rudin, C.
\newblock Learning optimized risk scores.
\newblock \emph{Journal of Machine Learning Research}, 20\penalty0
  (150):\penalty0 1--75, 2019.

\bibitem[Wang et~al.(2018)Wang, Jha, and Chaudhuri]{wang2018analyzing}
Wang, Y., Jha, S., and Chaudhuri, K.
\newblock Analyzing the robustness of nearest neighbors to adversarial
  examples.
\newblock In \emph{International Conference on Machine Learning}, pp.\
  5133--5142. PMLR, 2018.

\bibitem[Yang et~al.(2020{\natexlab{a}})Yang, Rashtchian, Wang, and
  Chaudhuri]{yang2020robustness}
Yang, Y.-Y., Rashtchian, C., Wang, Y., and Chaudhuri, K.
\newblock Robustness for non-parametric classification: A generic attack and
  defense.
\newblock In \emph{International Conference on Artificial Intelligence and
  Statistics}, pp.\  941--951. PMLR, 2020{\natexlab{a}}.

\bibitem[Yang et~al.(2020{\natexlab{b}})Yang, Rashtchian, Zhang, Salakhutdinov,
  and Chaudhuri]{yang2020adversarial}
Yang, Y.-Y., Rashtchian, C., Zhang, H., Salakhutdinov, R., and Chaudhuri, K.
\newblock Adversarial robustness through local lipschitzness.
\newblock \emph{arXiv preprint arXiv:2003.02460}, 2020{\natexlab{b}}.

\end{thebibliography}

\newpage
\appendix
\onecolumn
\section{Proofs}

\subsection{Separation and interpretability}\label{apx:separation_and_explainability}
\unboundedLowerBound*
\begin{proof}
Fix an integer $s$ and $\gamma>0.$ We will start by describing the dataset and prove it is linearly separable. We will then show that any decision tree of size $s$ must have accuracy smaller than $1/2+\gamma.$
\paragraph{Dataset.}
The dataset size is $n$, to be fixed later. 
The dataset includes, for any $i=1\dots n/4$, a group, $G_i$, of four points: two points $(i,-i+\epsilon),(i+\epsilon,-i)$ that are labeled positive and two points $(i,-i-\epsilon),(i-\epsilon,-i)$ that are labeled negative for some small $\epsilon>0$, see Figure~\ref{fig:unbounded_lower_bound}. 
\begin{figure}[!h]
\centering
    \includegraphics[width=0.3\textwidth]{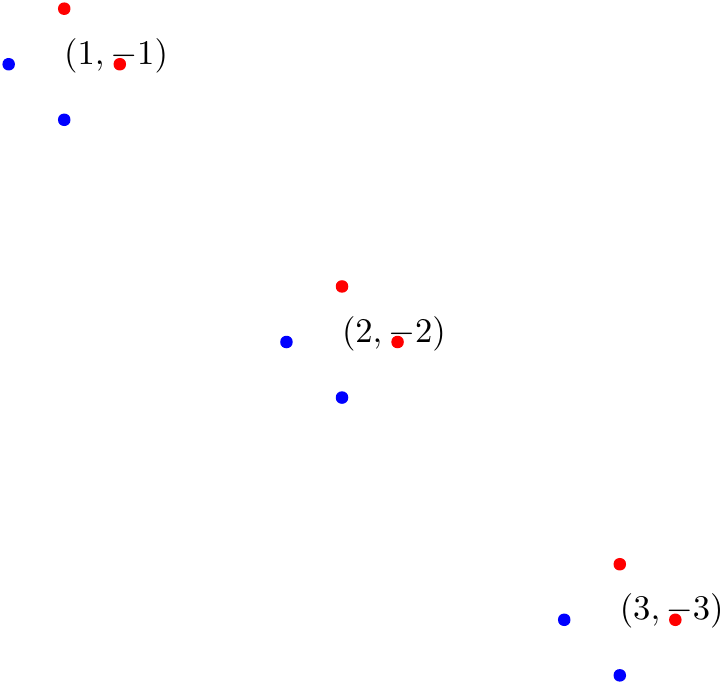}
    \caption{Proof of Theorem~\ref{thm:unbounded_lower_bound}. Linearly separable dataset that is not interpretable by a small decision tree. Around point $(i,-i)$ there are $4$ close points: two points $(i,-i+\epsilon),(i+\epsilon,-i)$ that are labeled positive and two points $(i,-i-\epsilon),(i-\epsilon,-i)$ that are labeled negative for some small $\epsilon>0$.}
    \label{fig:unbounded_lower_bound}
\end{figure}

To prove that the dataset is linearly separable focus on the vector $w=(0.5,0.5)$ with $|w|_1=1$. For each labeled examples $(x,y)$ in the dataset, the inner product is equal to $yw\cdot x=\epsilon/2>0$.

\paragraph{Accuracy.}

We will prove by induction that the points arriving in each node are a series of consecutive groups, perhaps except a few points that are in the ``tails". In each node, the number of positive and negative examples is about the same, so one cannot predict well. See intuition in Figure~\ref{fig:unbounded_lower_bound_projection}.

\begin{figure}[ht]
    \centering
    \includegraphics[scale=0.5]{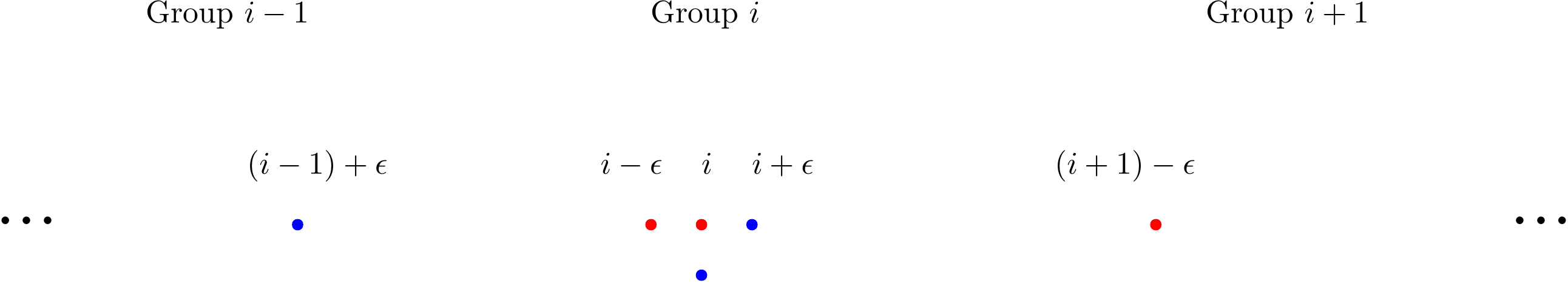}
    \caption{Projection of the points in the dataset to the first feature. We see that the groups appear one after another, each cut defined by an inner node in the tree, will leave the order between the groups as is. In a series of consecutive groups, the number of positive and negative examples is equal. Thus, the accuracy is close to $1/2$. Similar figure when projecting to the second feature.}
    \label{fig:unbounded_lower_bound_projection}
\end{figure}

More formally, a tail $G'_i$ is a subset of $G_i$, i.e.,  $G'_i\in\powerset(G_i)$, where $\powerset$ is the power set. 
We prove the following claim.  

 \textit{Claim.} For each inner node $v$, there are integers $j_0\leq j_1$ such that the points arriving to this node, $P_v$, satisfy $$P_v =  \cup_{i=j_0}^{j_1}G_{i}\cup G'_{j_0-1} \cup G'_{j_1+1},$$ 
 where $G'_{j_0-1}\in \powerset(G_{j_0-1})$ and $G'_{j_1+1}\in\powerset(G_{j_1+1}).$
 In other words, $P_v$ contains a series of consecutive groups and a few points from the group after and before the series.
 
  \begin{proof}
  We will prove the claim by induction on the level of the tree. The claim is correct for the root with $j_0=1$ and $j_1=n/4$. To prove the claim by induction, suppose an inner node uses a threshold $\theta$ and the first feature. Denote the closest integer to $\theta$ by  $j^*\leq \theta < j^*+1$. The points reaching the left son are
  $$\cup_{i=j_0}^{j^*}G_{i}\cup G_{j_0-1}'\cup G_{j^*+1}',$$ and the nodes reaching the right son are $$\cup_{i=j^*}^{j_1}G_{i}\cup G_{j^*-1}'\cup G_{j_1+1}'.$$
  A similar argument also hold if $v$ uses the second feature, which proves our claim. 
  \end{proof}

To finish the proof of Theorem~\ref{thm:unbounded_lower_bound}, first note that the number of positive and negative examples in each $G_i$ is exactly equal. Together with  the claim that we just proved, for each leaf $v$, the number of points that are correctly classified out of the points, $P_v$, reaching $v$, is at most $|P_v|/2+4.$ Thus the total number of points correctly classified by the entire tree is at most $n/2+4s.$ Thus, the accuracy of the tree is at most $1/2+4s/n.$ We take $n>4s/\gamma$ and get that the accuracy is smaller than $1/2+\gamma,$ which is what we wanted to prove. 

\end{proof}

\generalSepImpliesDt*
\begin{proof}
Each feature is in $[-1,1]$, so $r\leq 1.$
Fix $r \leq \Delta < 2r$ and $L=\lceil\frac{1}{\Delta}\rceil.$ We can bound $L$ by $$L\leq\frac{1}{\Delta}+1\leq\frac{1}{r}+1\leq\frac{2}{r}.$$
Take two data points, one labeled positive, $x^+$, and one negative $x^-$. By the $r$-separation, we know that there is a feature $i'$ such that 
$$ |x^+_{i'}-x^-_{i'}| \geq 2r>\Delta.$$ 
This means that we can find a threshold $\theta$ among the $2L+1$ thresholds  $ -L\cdot\Delta,\ldots, 0\cdot\Delta, 1\cdot\Delta,\ldots, L\cdot\Delta$ that distinguishes the examples $x^+$ and $x^-$, i.e., there is $j'\in\{-L,\ldots, L\}$, such that 
$$sign(x^+_{i'}-\Delta\cdot {j'})\neq sign(x^-_{i'}-\Delta\cdot {j'}).$$ 
We focus on the decision tree with all possible features and the $2L+1$ thresholds. All examples reaching the same leaf, has the same label. In other words, the training error is $0.$
Since there are $d\cdot (2L+1)$ pairs of feature and thresholds, the depth of the tree is at most $d(2L+1)\leq3dL\leq\frac{6d}{r}$. %

\end{proof}

\generalSepLowerDt*
\begin{proof}
We start with describing the dataset, then we will show that any decision tree must be large if we want to achieve $0.5+\gamma$ accuracy.

\paragraph{The dataset.} The inputs are all strings in $\{-1,1\}^d.$ The hypothesis is the parity function, i.e., and \begin{equation*}
f(x_1,\ldots,x_d) = \begin{cases} 1 &\mbox{if } \sum_{i=1}^d\frac{x_i+1}2 \equiv 1 \pmod{2}  \\
-1 & \mbox{o.w. }  \end{cases} 
\end{equation*}

\paragraph{Each node has equal number of positive and negative examples.} The main idea is that as long as the depth of a node is not $d$, it has exactly the same number of positive and negative examples reaching it. This is true since as long as at most $d-1$ features are fixed, there exactly the same number of positive and negative examples that agree on these features. 

\paragraph{Large size.} 
Denote by $N_1$ the set of all nodes that exactly one example reach them.  Denote this set size by $|N_1|=n_1.$ We want to prove that $n_1$ is large.  So far, we proved that for each node that contains more than one example, exactly half of the examples are labeled positive and half negative. 
Number of examples correctly classified is half of all the examples not in $N_1$ plus $n_1$. There are $2^d$ examples in total.
So the accuracy is equal to $\frac{(2^d-n_1)/2 + n_1}{2^d}=\frac12 + \frac{n_1}{2^{d+1}}$. The latter should be at least $\nicefrac12+\gamma$. Therefore, the size of the tree is at least $\gamma 2^d.$ 

\end{proof}

\subsection{Linear separability: weak learner}\label{apx:weak_learner}
We start with the more restricted case where the features are binary. This will give the necessary foundations for the general case, where features are in $[-1,1]$. In this case $\H_t$ is simplified to the set $\{ x\mapsto x_i : i\in[d]\}.$ 

\begin{restatable}{thm}{binaryWeakLearner}
\label{thm:binary_weak_learner}
For any data %
in 
$\{-1,1\}^d\times\{-1,1\}$ that is labeled by a $\gamma$-linearly separable a hypothesis $f$ and for any distribution $\mu$ on the examples, there is hypothesis $h\in\H_t$  such that $$\Pr_{x\sim \mu}(h(x) = f(x)) \geq  \frac12+\frac\gamma2.$$
\end{restatable}
\begin{proof}%
The proof's high-level idea is to represent the class as a bipartite graph and lower bound the weighted density of this graph. The high lower-bound leads to a high degree vertex, which will correspond to our desired weak learner. 

Recall that by the fact that the data is $\gamma$-linearly separable, we know that there is a vector $w$ with $|w|_1=1$ such that for eac labeled example $(x,y)$ it holds that 
\begin{equation}\label{eq:binary_weak_learner_separator}
  yw\cdot x\geq\gamma.  
\end{equation}

\paragraph{Bipartite graph description.}
Consider the following bipartite graph. The vertices are the $m$ examples in the training data and the $d$ hypotheses in the binary version of $\H_t$. There is an edge $(x,h_i)\in E$ between an example $x$ and hypothesis $h_i$, if it correctly classify $x$, i.e., if $f(x)=h_i(x)$. Each vertex is given a weight: example $x^j$ gets weight $\mu_j$ and a hypothesis $h_i$ gets weight $w_i$, where $w$ is in Equation~(\ref{eq:binary_weak_learner_separator}). 

\paragraph{Assumption: $\sum_{i=1}^d w_i=1$.}  We assume without loss of generality that $w_i\geq0$, otherwise if there is $w_i<0$, we can multiply the $i$-th feature in all the examples by $-1$. After this multiplication, the linearity assumption still holds.  
We know that $\sum_{i=1}^d |w_i|=1$, and since $w_i\geq 0$, we get that $\sum_{i=1}^d w_i=1$.

\paragraph{Lower bound weighted-density.}
The main idea is to lower bound the weighted density $\rho$ of the bipartite graph, which is the sum over all edges $(h_{i},x^j)$ in the graph, each with weight $w_i,\mu_j$: 
$$\rho=\sum_{(x^j,h_{i})\in E}w_i\mu_j.$$

To prove a lower bound on $\rho$, we focus on one labeled example $(x^j,y^j).$ 
From the linearity assumption, Equation~(\ref{eq:binary_weak_learner_separator}), we know that $\sum_{i=1}^dy^jw_ix^j_i\geq\gamma.$
Recall that $y^jx^j_i$ is equal to $+1$ or $-1$, since we are in the binary case. 
We can separate the sum depending on whether $y^jx^j_i$ is equal to $+1$ or $-1$ and get that $$\sum_{i:y^jx^j_i=1}w_i - \sum_{i:y^jx^j_i=-1}w_i\geq\gamma.$$
We know that $\sum_{i=1}^d w_i=1$, thus $\sum_{i:y^jx^j_i=-1}w_i = 1 - \sum_{i:y^jx^j_i=1}w_i$, and the inequality can be rewritten as  
$$2\sum_{i:y^jx^j_i=1}w_i - 1\geq\gamma.$$
Notice that $(x^j,h_i)\in E \Leftrightarrow y^jx^j_i=1$. Thus, the inequality can be further rewritten as 
$$\sum_{i:(x^j,h_i)\in E}w_i\geq\frac{1+\gamma}{2}.$$
This inequality holds for any labeled example, so we can sum all these inequalities, each with weight $\mu_j$ and get that  $$\rho\geq (1+\gamma)/2.$$

\paragraph{Finding a weak learner.} Since $w$ is a probability distribution, we can rewrite $\rho$ and get  $$\E_{i\sim w}\left[\sum_{j:(x^j,h_i)\in E}\mu_j\right]\geq (1+\gamma)/2.$$ 
From the probabilistic method \cite{alon2004probabilistic}, there is a hypothesis $h_i$ such that $$\sum_{j:(x^j,h_i)\in E}\mu_j\geq (1+\gamma)/2.$$ By the definition of the graph, $(x^j,h_i)\in E \Leftrightarrow h_i(x^j) = f(x^j)$. Thus, we get that  $$\Pr_{x\sim \mu}(h_i(x) = f(x)) \geq  \frac12+\frac\gamma2,$$ which is exactly what we wanted to prove.

\end{proof}

\begin{restatable}[binary weak-learner]{thm}{binaryWeakLearnerDis}
\label{thm:binary_weak_learner_dis}
For any distribution $\mu$ over labeled examples $\{-1,+1\}^d\times\{-1,+1\}$ that satisfies linear separability with a $\gamma$-margin, and for any $\delta\in(0,1)$ there is $m=O\left(\frac{d+\log\frac1\delta}{\gamma^2}\right)$, such that with probability at least $1-\delta$ over the sample $S$ of size $m$, it holds that $$\Pr_{(x,y)\sim\mu}(h_S(x)=y)\geq \frac12+\frac\gamma4.$$  
\end{restatable}

\begin{proof}
We start with a known fact\footnote{To bound the $VC(\H_t)$ note that for any $d+1$ points in $\R^d$ there is at least one point $v$, where in each coordinate it is not the largest one among the $d+1$ points. Thus, it's impossible that $v$ is labeled $+1$ while all the rest of the points are labeled $-1$.} that 
$$VC(\H_t)\leq d.$$

Denote the best hypothesis in the binary version in $\H_t$ as $h^*$. From Theorem~\ref{thm:binary_weak_learner}, we know that $\Pr_x(h^*(x) = f(x))\geq \frac12+\frac{\gamma}2.$
For every sample $S$, denote by $h_S$ the hypothesis in the binary version of $\H_t$ that optimizes the accuracy on the sample $S.$ 
From the fundamental theorem of statistical learning \cite{shalev2014understanding}, we know that for  $m=O\left(\frac{d+\log\frac1\delta}{(\gamma/4)^2}\right)$, with probability at least $1-\delta$ over the sample $S$ of size $m$, it holds that 
$$\Pr_{(x,y)\sim\mu}(h_S(x) = y) \geq \Pr_{(x,y)\sim\mu}(h^*(x) = y) -  \frac\gamma4\geq\frac12+\frac\gamma4.$$
\end{proof}

\WeakLearner*
\begin{proof}%
The proof is similar in spirit to the proof of Theorem~\ref{thm:binary_weak_learner}.
The main difference is the technique to lower bound $\rho$, the weighted density. In the current proof we use the fact that if for some positive example $x$, its $i$-th feature, $x_i$, is high, then it contains many edges in the graph. For a negative example $x$, if its $i$-th feature, $x_i$, is low, then it contains many edges in the graph.

\paragraph{Bipartite graph description.}
We first discretize the segment $[-1,1]$ to $\ell\geq2/\alpha+1$ values with $\Delta=2/(\ell-1)$:
$$Z=\{-1, -1+\Delta,\ldots,1-\Delta,1\}.$$ 
The value of $\Delta$ was chosen such that $|Z|=\ell.$
We focus on the following subclass of $\H'_t$ which is a discretization of $\H_t$
$$\H'_t = \{h_{i,\theta}\}\quad \text{ for } i\in [d], \theta\in Z.$$
We use a similar bipartite graph as in the proof of Theorem~\ref{thm:binary_weak_learner} for the subclass $\H'_t$: 
the vertices are the $m$ examples and the hypotheses in $\H'_t$; and there is an edge $(x,h_{i,\theta})\in E$ between an example $x$ with label $y$ and hypothesis $h_{i,\theta}$ whenever $yx_i\geq\theta$, i.e., when $h_{i,\theta}$ correctly classify $x$.

Recall that by the fact that the data is $\gamma$-linearly separable, we know that there is a vector $w$ with $|w|_1=1$ and for any labeled example $(x,y)$ it holds that 
\begin{equation}\label{eq:weak_learner_separator}
  yw\cdot x\geq\gamma.  
\end{equation}
From the same argument as in Theorem~\ref{thm:binary_weak_learner}, we assume that $\sum_{i=1}^d w_i=1.$

We are now ready to give each vertex in the bipartite graph a weight: example $x^j$ gets weight $\mu_j$ and a hypothesis $h_{i,\theta}$ gets weight $w_{i,\theta} = w_i/\ell$, where $w$ is as in Equation~(\ref{eq:weak_learner_separator}). The weights on the hypotheses were chosen such that they will sum up to $1$: 
\begin{equation}\label{eq:weak_learner_distrbution_over_hypotheses}
\sum_{i,\theta}w_{i,\theta}=\sum_{i,\theta}\frac{w_{i}}{\ell}=\sum_i w_i= 1.
\end{equation}

\paragraph{Lower bound weighted-density.}
We will lower bound the weighted density $\rho$ of the bipartite graph, which is the sum over all edges $(x^j,h_{i,\theta})$ in the graph, each with weight $w_{i,\theta}\mu_j$: 
$$\rho=\sum_{(x^j,h_{i,\theta})\in E}w_{i,\theta}\mu_j.$$
To show a lower bound on $\rho$, we focus on one labeled example $(x,y)$ and one feature $i$ and we will lower bound the following sum  
\begin{equation}\label{eq:weak_learner_fix_example_fix_feature}
\sum_{\theta:(x,h_{i,\theta})\in E}w_{i,\theta}=\frac{w_i}{\ell}\sum_{\theta:(x,h_{i,\theta})\in E}1.
\end{equation}
The sum in the RHS is equal to the number of $\theta$'s such that $yx^j_i\geq\theta$:
\begin{equation}
 d_{x,i}=|\{\theta \in Z : yx_i\geq\theta\}|.   
\end{equation}
To lower bound $d_{x,i}$ we separate the analysis depending on the label $y$.  
If $x$ is a positive example, i.e., $y=+1$, we have that   $$d_{x,i}\geq (1+x_i-\Delta)/\Delta$$
For a negative example $x$, we can lower bound $d_{x,i}$ by 
$$d_{x,i}\geq (1-x_i-\Delta)/\Delta$$
To summarize these two equations we get that 
$$d_{x,i}\geq (1+yx_i-\Delta)/\Delta$$
Going back to Equation~(\ref{eq:weak_learner_fix_example_fix_feature}),
$$\sum_{\theta:(x,h_{i,\theta})\in E}w_{i,\theta}\geq\frac{w_i}{\ell}\cdot\frac{1+yx_i-\Delta}{\Delta}=w_i(1+yx_i-\Delta)\cdot \frac{\ell-1}{2\ell}.$$
Summing over all features 
\begin{eqnarray}
\sum_{i,\theta:(x,h_{i,\theta})\in E}w_{i,\theta} &=& \frac{\ell-1}{2\ell}\left(\sum_i w_i(1-\Delta) +\sum_iyx_iw_i\right)\\&\geq&\frac{\ell-1}{2\ell}(1+\gamma-\Delta)
=\left(\frac12+\frac{\gamma}{2}-\frac\Delta2\right)\left(1-\frac1\ell\right)
\end{eqnarray}

Weighted sum over all examples yields the desired lower bound on $\rho$  
$$\rho\geq\left(\frac12+\frac{\gamma}{2}-\frac\Delta2\right)\left(1-\frac1\ell\right)\geq \frac{1}{2}+\frac\gamma2-\alpha.$$

\paragraph{Finding a weak learner.}
The proof now continues similarly to the proof of  Theorem~\ref{thm:binary_weak_learner}. 
Since $w_{i,\theta}$ is a probability distribution over all hypotheses in $\H'_t$ (see Equation~\ref{eq:weak_learner_distrbution_over_hypotheses}), we can rewrite $\rho$ and get  $$\E_{h_{i,\theta}\sim w}\left[\sum_{j:(x^j,h_{i,\theta})\in E}\mu_j\right]\geq \frac{1}{2}+\frac\gamma2-\alpha.$$ 
From the probabilistic method \cite{alon2004probabilistic}, there is hypothesis $h_i$ such that $$\sum_{j:(x^j,h_i)\in E}\mu_j\geq \frac{1}{2}+\frac\gamma2-\alpha.$$ By the definition of the graph, $(x^j,h_i)\in E \Leftrightarrow h_i(x^j) = f(x^j)$. Thus, we get that  $$\Pr_{x\sim \mu}(h_i(x) = f(x)) \geq  \frac{1}{2}+\frac\gamma2-\alpha,$$ which is exactly what we wanted to prove.

\end{proof}

\WeakLearnerDis*
\begin{proof}
Proof is similar to Theorem~\ref{thm:binary_weak_learner_dis}.
\end{proof}

\subsection{Linear separability: risk scores}\label{apx:risk_score}
\RsMonotone*
\begin{proof}%
Fix a risk score model $f$ which is defined by a series of $m$ conditions $``x_{i_1}\geq\theta_1",\ldots, ``x_{i_m}\geq\theta_m"$ and weights $w_0, w_1,\ldots, w_m$.
The score, $s(z)$, of an example $z$ is the weighted-sum of all satisfied conditions, $$s(z)=w_0+\sum_{j=1}^m w_jI_{z_j \geq \theta_j},$$ where $I_A=1$ if $A$ is true, and $I_A=0$ otherwise. The prediction of the model $f$ is equal to %
$$f(z)=sign(s(z)).$$

Fix examples $x,y\in\mathbb{R}^d$ with $x\leq y.$ Our goal is to show that $f(x)\leq f(y).$ The key observation is that any condition $x_{i_j}\geq\theta_j$ satisfied by $x$ is also satisfied by $y$ because $y_{i_j}\geq x_{i_j}\geq\theta_j,$ by our assumption that $x\leq y.$ In different words we have that 
\begin{equation}\label{eq:risk_score_monotone}
   I_{y_j \geq \theta_j}\geq I_{x_j \geq \theta_j}. 
\end{equation}

This implies that the score of $y$ is at least the score of $x$, since it holds that 
$$s(y)=w_0\sum_{j=1}^m w_jI_{y_j \geq \theta_j}\geq w_0+\sum_{j=1}^m w_j I_{x_j \geq \theta_j}=s(x).$$
Thus, we get exactly what we wanted to prove $f(y)=sign(s(y))\geq sign(s(x))=f(x).$

As an aside, at this point, it should become apparent why we restricted our conditions to be of the form $``x_{i_j}\geq\theta_j"$ and did not allow natural conditions of the form $``x_{i_j}\leq\theta_j"$, such conditions will not be monotone and Inequality~(\ref{eq:risk_score_monotone}) will not hold. 
\end{proof}

\MonotoneImpliesRobust*
\begin{proof}%
The high-level idea is to focus on a noisier distribution than the original one. The noise is small enough so that the new data will remain linearly separable with a (slightly worse) margin. Therefore, the learning algorithm can be applied. We call the learning algorithm with a sample from the noisy distribution and return its result. We will prove that a point correctly classified in the noisy dataset implies that the noiseless point is robust to adversarial examples.

\paragraph{Noisy data.}
Fix a data $D\subseteq \X\times\{-1,+1\}$ and a distribution $\mu$ on $D$. We will create a noisy distribution $\mu'$ 
on the labeled examples by mapping each example $(x,y)\in D$ to a new labeled example $(x',y)$ where
$$x'=x-y\frac\gamma2\mathbf{1},$$
where $\mathbf{1}$ is the vector of all $1.$
Thus, if $x$ is a positive labeled example ($y=+1$) then  $x'=x-\frac\gamma2\mathbf{1}$. This means that from each coordinate we decrease $\gamma/2$. Intuitively, we make $x$ looks more negative. Similarly, if $x$ is a negative labeled example ($y=-1$) we create a new example $x'=x+\frac\gamma2\mathbf{1}$, intuitively, making $x$ looks more positive by adding $\gamma/2$ to each coordinate.  

If we have a samples from $\mu$, then we can easily sample from $\mu'$ by subtracting $y\frac\gamma2\mathbf{1}$ from each labeled example $(x,y)$ in the sample.

\paragraph{Noisy data  is $\gamma/2$-linearly separable.}
Suppose that the original data $D$ is $\gamma$-separable. This means that there is a vector $w$ with $|w|_1=1$ such that for every labeled example $(x,y)\in D$ in it holds that $ y w\cdot x\geq \gamma.$
We know that the corresponding example in the noisy data is equal to $x'=x-y\frac\gamma2\mathbf{1}$. We will prove a lower bound on $yw\cdot x'$ and this will prove that the noisy input is also linearly separable with a margin. We will use the fact that $y^2=1$ for any label $y$ and get that 
$$yw\cdot x' = yw\left(x-y\frac\gamma2\mathbf{1}\right)=ywx-y^2w\cdot \frac\gamma2\mathbf{1} = ywx-w\cdot \frac\gamma2\mathbf{1}$$
Recall that $|w|_1=1$, which means that  $w\cdot\mathbf{1}=\sum_i w_i\leq \sum_i |w_i| = |w|_1 = 1.$ This means that $-w\cdot\frac\gamma2\mathbf{1}\geq -\gamma/2.$
Together with the fact that $yw\cdot x\geq\gamma$, we can now give the desired lower bounds on $yw\cdot x'$
$$yw\cdot x' \geq ywx - \gamma/2\geq \gamma-\gamma/2=\gamma/2.$$
In different words, the new data is $\gamma/2$-linearly separable.

\paragraph{Model is robust.}
In order to construct a robust model, we take our sample $S$ from $\mu$. Then we transform it to a sample from $\mu'$ by subtracting noise of $y\frac\gamma2\mathbf{1}$ to each labeled example $(x,y)$, and then we call algorithm $A$ with the noisy training data. 
From our assumption, the resulting model has accuracy $1-\epsilon(\frac\gamma 2)$. We will show that the returned model has astuteness of $1-\epsilon(\frac\gamma 2)$ at radius $\gamma/2$ with respect to $\mu.$ We do this by proving that if $(x',y)$ is correctly classified than the model is robust at $x$ with radius $\gamma/2.$ 

Fix a noisy labeled example $(x',y)$ and an example $z$ that is $\gamma/2$ close to $x$, in $\ell_{\infty}$, i.e., $\|x-z\|_\infty\leq\gamma/2$. 
 If $x$ is positively labeled then $x'$ is also positively labeled, and from the construction of the noisy dataset it holds that 
 $$ x'\leq z.$$ 
 Hence, from the monotone property of the model, if $x'$ is labeled correctly then so is $z$. A similar argument holds if $x$ is negatively labeled.

\end{proof}

\section{Additional Experiment Details}
\label{app:experiment}

\subsection{Setups}

The experiments are performed on a Intel Core i9 9940X machine with 128GB of RAM.
The code for the experiments is available at \url{https://github.com/yangarbiter/interpretable-robust-trees}.

\paragraph{Additional dataset details.}
The adult, bank, breastcancer, mammo, mushroom, and spambase datasets are
retrieved from a publicly available repository 
\footnote{\url{https://github.com/ustunb/risk-slim/tree/master/examples/data}},
and these datasets are used by \citet{ustun2019learning}.
The careval, ficobin, and campasbin datasets are also retrieved from a publicly available 
source\footnote{\url{https://github.com/Jimmy-Lin/GeneralizedOptimalSparseDecisionTrees/tree/master/experiments/datasets}} and  used by~\citet{lin2020generalized}.
We also added the diabetes, heart, and ionosphere dataset from~\footnote{\url{https://www.csie.ntu.edu.tw/~cjlin/libsvmtools/datasets/}} and
bank2 dataset from~\citet{moro2014data}.
All features are scaled to $[0, 1]$ by the following formula $(x – min) / (max – min)$, where 
$x$ represents the feature value, and $min$ and $max$ represents the minimum and maximum value
of the given feature across the entire data.

In the adult dataset, the target is to predict whether the person's income is greater then $50,000$.
For bank and bank2 datasets, we want to predict whether the client opens a bank account after a marketing call.
In the breastcancer dataset, we want to predict whether the given sample is benign.
In the heart dataset, we want to detect the presence of heart disease in a patient. 
In the mammo dataset, we want to predict whether the sample from the mammography is malignant.
In the mushroom dataset, we want to predict whether the mushroom is poisonous.
In the spambase dataset, we want to predict whether an email is a spam.
In the careval dataset, the goal is to evaluate cars.  
In the ficobin dataset,  we want to predict a person credit risk.
In the campasbin dataset, we want to predict whether a convicted criminal will re-offend again.
In the diabetes dataset, we want to predict whether or not the patients in the dataset have diabetes.
In the ionosphere dataset, we want to predict whether the radar data are showing some evidence of some type of structure in the ionosphere.

\paragraph{Baseline implementations}
For DT, we use the implementation from \texttt{scikit-learn}~\cite{scikit-learn}
and set the splitting criteria to entropy.
For LCPA and RobDT, we use the implementation from their
original authors\footnote{\url{https://github.com/ustunb/risk-slim} and \url{https://github.com/chenhongge/RobustTrees}}.

\paragraph{Measure empirical robustness.}
The IR for DT and RobDT can be measured using the method in 
\cite{kantchelian2016evasion,yang2020robustness}.
The IR for BBM-RS is measured using the method in
\cite{andriushchenko2019provably}.
The IR for LCPA can be measured by solving a
linear program.

\section{Additional Results}\label{apx:additional_results}

\subsection{Examples that are similar but labeled differently}

The compasbin dataset has the lowest $r$-separateness. Its binary features are:
\begin{itemize}
    \setlength\itemsep{-.2em}
    \item sex:Female
    \item age:$<21$
    \item age:$<23$
    \item age:$<26$
    \item age:$<46$
    \item juvenile-felonies:=0
    \item juvenile-misdemeanors:=0
    \item juvenile-crimes:=0
    \item priors:=0
    \item priors:=1
    \item priors:2-3
    \item priors:$>$3
\end{itemize}

There are $852$ people who are: male, age between $26$ to $46$, did not commit any juvenile felonies, misdemeanors, and crimes, and have more than $3$ previous criminal conviction.
These people will have the same feature vector while for their labels,
$542$ recidivate within two years while $310$ people did not.

\subsection{Relationship between explainability, accuracy, and robustness in BBM-RS}

To understand the interaction between explainability, accuracy, and robustness,
we measure these criteria of BBM-RS with different noise levels $\tau$.
The results are shown in Figure \ref{fig:tradeoff_noise_bbm}.
We observed that, by changing the noise level,
that robustness and the explanation complexity go hand in hand.
For higher noise levels, we have a higher robustness and lower explanation complexity and accuracy.
The shows that by making the model simpler, we can have better robustness and explainability while loosing
some accuracy.

\begin{figure}[ht!]
    \centering
    \subfloat[adult]{\includegraphics[width=.24\textwidth]{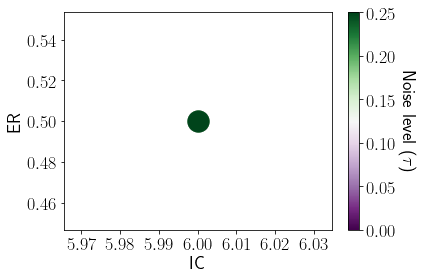}}
    \subfloat[bank]{\includegraphics[width=.24\textwidth]{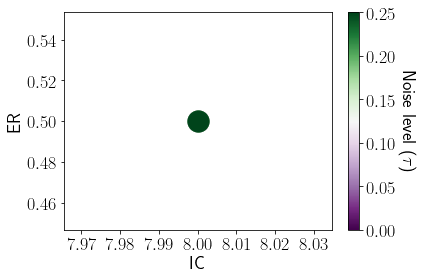}}
    \subfloat[bank2]{\includegraphics[width=.24\textwidth]{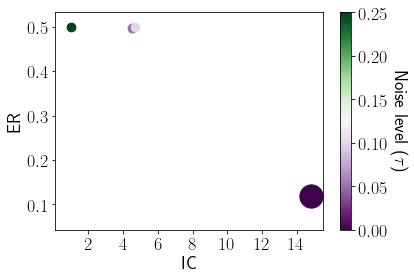}}
    \subfloat[breastcancer]{\includegraphics[width=.24\textwidth]{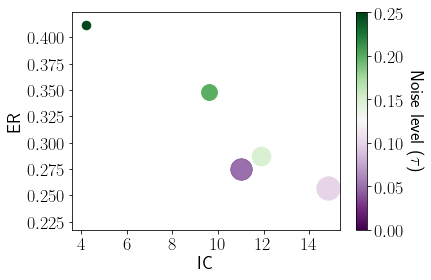}}

    \subfloat[careval]{\includegraphics[width=.24\textwidth]{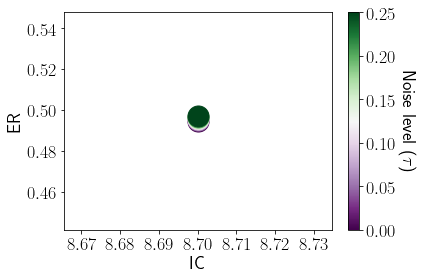}}
    \subfloat[compasbin]{\includegraphics[width=.24\textwidth]{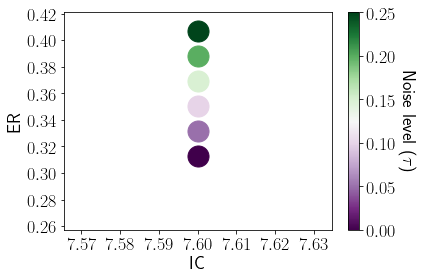}}
    \subfloat[diabetes]{\includegraphics[width=.24\textwidth]{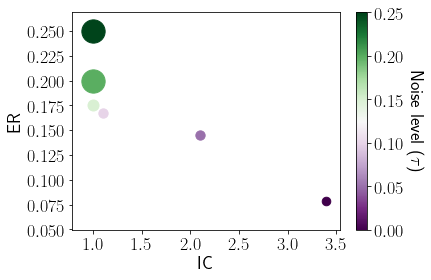}}
    \subfloat[ficobin]{\includegraphics[width=.24\textwidth]{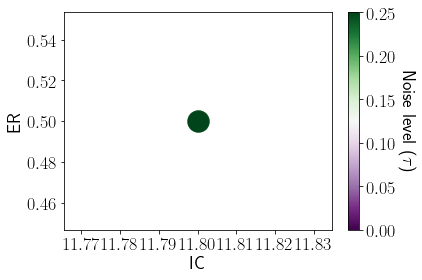}}

    \subfloat[heart]{\includegraphics[width=.24\textwidth]{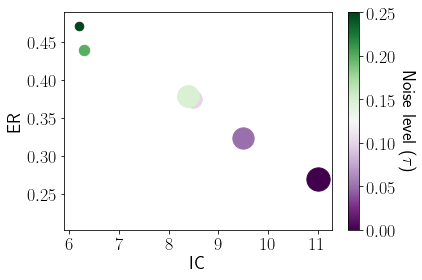}}
    \subfloat[ionosphere]{\includegraphics[width=.24\textwidth]{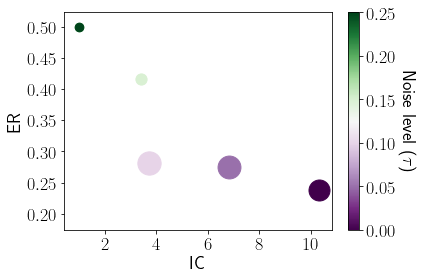}}
    \subfloat[mammo]{\includegraphics[width=.24\textwidth]{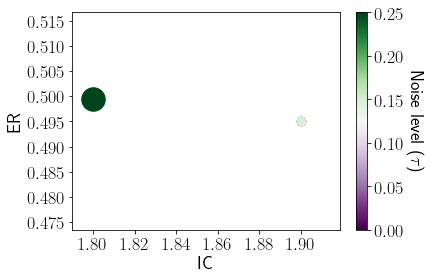}}
    \subfloat[mushroom]{\includegraphics[width=.24\textwidth]{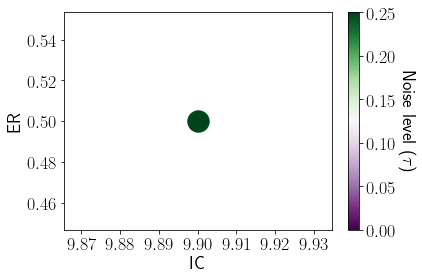}}

    \subfloat[spambase]{\includegraphics[width=.24\textwidth]{figs/bbm_tradeoffs/spambase.png}}

    \caption{Trade-off between explainability and accuracy for BBM-RS.
    The size of the ball represents the accuracy.
    }
    \label{fig:tradeoff_noise_bbm}
\end{figure}

\subsection{Trade-off between explanation complexity and accuracy for BBM-RS}

To understand how explanation complexity effects accuracy, we first train a BBM-RS classifier.
The learned BBM-RS classifier consists of $T$ weak learners.
We then measure the test accuracy of using only $i$ weak learners for prediction, where
$i=1 \ldots T$.
Finally, we plot out the figure of accuracy versus explanation complexity
(number of unique weak learners) in Table~\ref{fig:tradeoff_bbm}.
Note that for the same explanation complexity, there may be more then one test accuracy.
In this case we show the highest test accuracy.
In Table~\ref{fig:tradeoff_bbm}, we see that with the increase of explanation complexity,
generally the test accuracy also increases.

\begin{figure}[ht!]
    \centering
    \subfloat[adult]{\includegraphics[width=.24\textwidth]{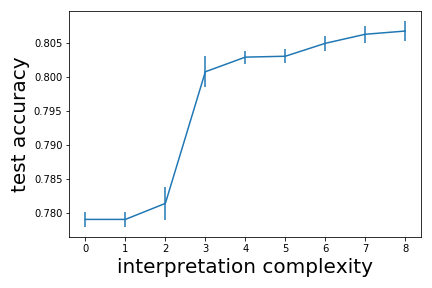}}
    \subfloat[bank]{\includegraphics[width=.24\textwidth]{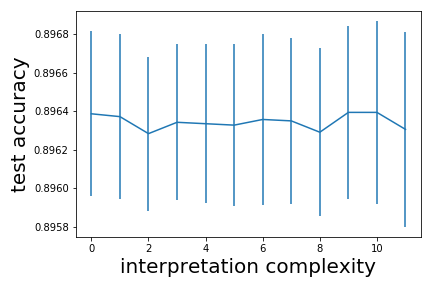}}
    \subfloat[bank2]{\includegraphics[width=.24\textwidth]{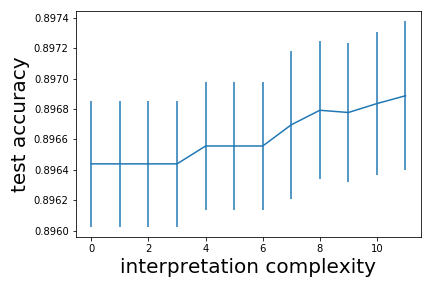}}
    \subfloat[breastcancer]{\includegraphics[width=.24\textwidth]{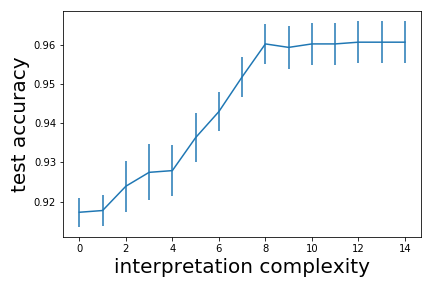}}

    \subfloat[careval]{\includegraphics[width=.24\textwidth]{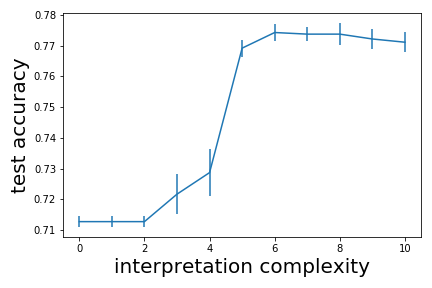}}
    \subfloat[compasbin]{\includegraphics[width=.24\textwidth]{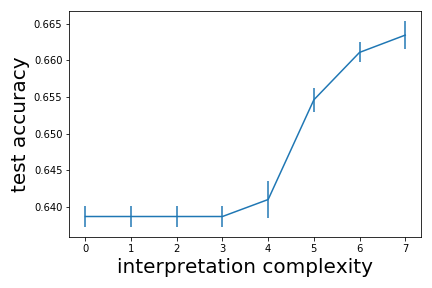}}
    \subfloat[diabetes]{\includegraphics[width=.24\textwidth]{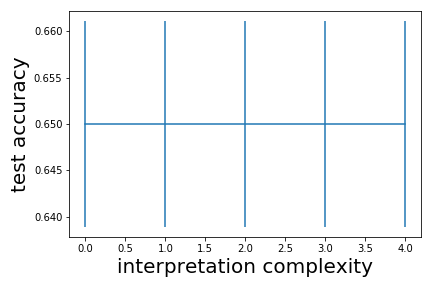}}
    \subfloat[ficobin]{\includegraphics[width=.24\textwidth]{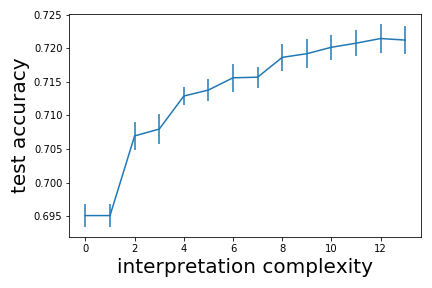}}

    \subfloat[heart]{\includegraphics[width=.24\textwidth]{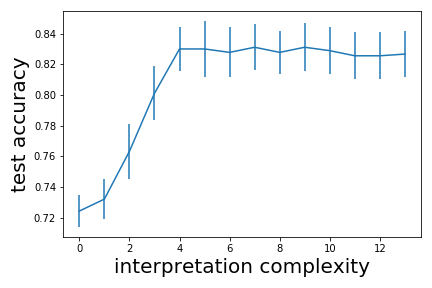}}
    \subfloat[ionosphere]{\includegraphics[width=.24\textwidth]{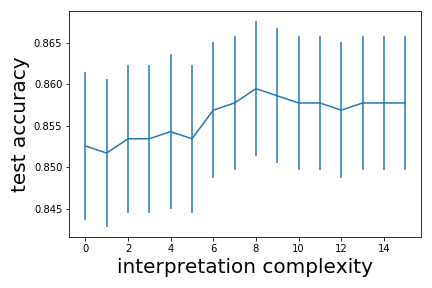}}
    \subfloat[mammo]{\includegraphics[width=.24\textwidth]{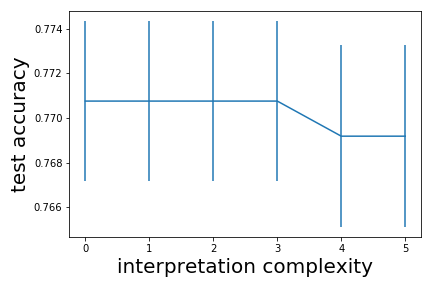}}
    \subfloat[mushroom]{\includegraphics[width=.24\textwidth]{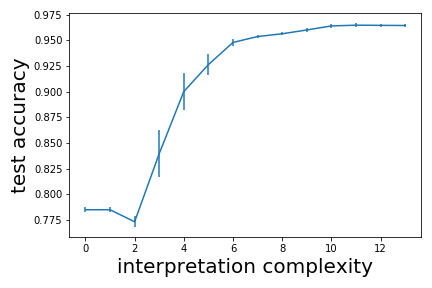}}

    \subfloat[spambase]{\includegraphics[width=.24\textwidth]{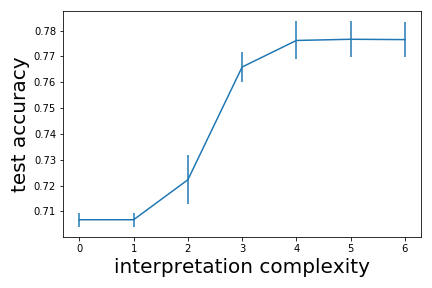}}

    \caption{Trade-off between explanation complexity and test accuracy for BBM-RS.}
    \label{fig:tradeoff_bbm}
\end{figure}

\begin{table*}[ht]
    \tiny
    \setlength{\tabcolsep}{2pt}
    \centering
\begin{tabular}{lcccc|cccc|cccc}
\toprule
{} & \multicolumn{4}{c}{EC} & \multicolumn{4}{c}{test accuracy} & \multicolumn{4}{c}{ER} \\
{} &                DT &             RobDT &            LCPA &            BBM-RS &               DT &            RobDT &            LCPA &           BBM-RS &               DT &            RobDT &            LCPA &           BBM-RS \\
\midrule
adult        &  $414.20 \pm 5.66$ &       $287.90 \pm 35.66$ &         $14.90 \pm 1.46$ &   $\mathbf{6.00 \pm .60}$ &  $\mathbf{0.83 \pm .00}$ &  $\mathbf{0.83 \pm .00}$ &           $0.82 \pm .00$ &           $0.81 \pm .00$ &  $\mathbf{0.50 \pm .00}$ &  $\mathbf{0.50 \pm .00}$ &  $0.12 \pm .02$ &  $\mathbf{0.50 \pm .00}$ \\
bank         &    $30.70 \pm .15$ &          $26.80 \pm .20$ &           $8.90 \pm .66$ &  $\mathbf{8.00 \pm 1.41}$ &  $\mathbf{0.90 \pm .00}$ &  $\mathbf{0.90 \pm .00}$ &  $\mathbf{0.90 \pm .00}$ &  $\mathbf{0.90 \pm .00}$ &  $\mathbf{0.50 \pm .00}$ &  $\mathbf{0.50 \pm .00}$ &  $0.20 \pm .03$ &  $\mathbf{0.50 \pm .00}$ \\
bank2        &    $30.00 \pm .30$ &          $30.70 \pm .15$ &         $13.80 \pm 1.54$ &  $\mathbf{4.50 \pm 1.34}$ &  $\mathbf{0.91 \pm .00}$ &           $0.90 \pm .00$ &           $0.90 \pm .00$ &           $0.90 \pm .00$ &           $0.12 \pm .01$ &           $0.18 \pm .02$ &  $0.10 \pm .01$ &  $\mathbf{0.50 \pm .00}$ \\
breastcancer &   $15.20 \pm 1.25$ &           $7.40 \pm .60$ &  $\mathbf{6.00 \pm .00}$ &           $11.00 \pm .89$ &           $0.94 \pm .00$ &           $0.94 \pm .01$ &  $\mathbf{0.96 \pm .00}$ &  $\mathbf{0.96 \pm .01}$ &           $0.23 \pm .01$ &  $\mathbf{0.29 \pm .01}$ &  $0.28 \pm .00$ &           $0.27 \pm .01$ \\
careval      &   $59.30 \pm 2.22$ &          $28.20 \pm .65$ &          $10.10 \pm .97$ &   $\mathbf{8.70 \pm .47}$ &  $\mathbf{0.97 \pm .00}$ &           $0.96 \pm .00$ &           $0.91 \pm .01$ &           $0.77 \pm .00$ &  $\mathbf{0.50 \pm .00}$ &  $\mathbf{0.50 \pm .00}$ &  $0.19 \pm .02$ &  $\mathbf{0.50 \pm .00}$ \\
compasbin    &  $67.80 \pm 13.01$ &         $33.70 \pm 3.05$ &  $\mathbf{5.40 \pm .22}$ &            $7.60 \pm .16$ &  $\mathbf{0.67 \pm .00}$ &  $\mathbf{0.67 \pm .00}$ &           $0.65 \pm .00$ &           $0.66 \pm .00$ &  $\mathbf{0.50 \pm .00}$ &  $\mathbf{0.50 \pm .00}$ &  $0.15 \pm .01$ &           $0.33 \pm .01$ \\
diabetes     &   $31.20 \pm 6.96$ &         $27.90 \pm 2.95$ &           $6.00 \pm .00$ &   $\mathbf{2.10 \pm .53}$ &           $0.74 \pm .01$ &           $0.73 \pm .01$ &  $\mathbf{0.76 \pm .01}$ &           $0.65 \pm .01$ &           $0.08 \pm .01$ &           $0.08 \pm .00$ &  $0.09 \pm .00$ &  $\mathbf{0.15 \pm .05}$ \\
ficobin      &    $30.60 \pm .22$ &        $59.60 \pm 29.82$ &  $\mathbf{6.40 \pm .16}$ &           $11.80 \pm .65$ &           $0.71 \pm .00$ &           $0.71 \pm .00$ &           $0.71 \pm .00$ &  $\mathbf{0.72 \pm .00}$ &  $\mathbf{0.50 \pm .00}$ &  $\mathbf{0.50 \pm .00}$ &  $0.22 \pm .01$ &  $\mathbf{0.50 \pm .00}$ \\
heart        &   $20.30 \pm 1.60$ &          $13.60 \pm .88$ &         $11.90 \pm 1.46$ &   $\mathbf{9.50 \pm .82}$ &           $0.76 \pm .01$ &           $0.79 \pm .01$ &  $\mathbf{0.82 \pm .01}$ &  $\mathbf{0.82 \pm .01}$ &           $0.23 \pm .02$ &           $0.31 \pm .02$ &  $0.14 \pm .01$ &  $\mathbf{0.32 \pm .02}$ \\
ionosphere   &    $11.30 \pm .98$ &           $8.60 \pm .76$ &         $17.90 \pm 3.14$ &  $\mathbf{6.80 \pm 1.96}$ &           $0.89 \pm .01$ &  $\mathbf{0.92 \pm .01}$ &           $0.88 \pm .01$ &           $0.86 \pm .01$ &           $0.15 \pm .01$ &           $0.25 \pm .01$ &  $0.07 \pm .01$ &  $\mathbf{0.28 \pm .01}$ \\
mammo        &   $27.40 \pm 5.09$ &          $12.40 \pm .65$ &           $7.20 \pm .65$ &   $\mathbf{1.90 \pm .60}$ &  $\mathbf{0.79 \pm .00}$ &  $\mathbf{0.79 \pm .00}$ &  $\mathbf{0.79 \pm .00}$ &           $0.77 \pm .00$ &           $0.47 \pm .01$ &  $\mathbf{0.50 \pm .00}$ &  $0.21 \pm .02$ &  $\mathbf{0.50 \pm .00}$ \\
mushroom     &    $10.80 \pm .25$ &  $\mathbf{9.10 \pm .10}$ &         $23.80 \pm 1.50$ &            $9.90 \pm .89$ &  $\mathbf{1.00 \pm .00}$ &  $\mathbf{1.00 \pm .00}$ &  $\mathbf{1.00 \pm .00}$ &           $0.97 \pm .00$ &  $\mathbf{0.50 \pm .00}$ &  $\mathbf{0.50 \pm .00}$ &  $0.10 \pm .01$ &  $\mathbf{0.50 \pm .00}$ \\
spambase     &  $153.90 \pm 8.51$ &         $72.30 \pm 2.89$ &          $29.50 \pm .76$ &   $\mathbf{5.60 \pm .48}$ &  $\mathbf{0.92 \pm .00}$ &           $0.87 \pm .00$ &           $0.88 \pm .00$ &           $0.79 \pm .01$ &           $0.00 \pm .00$ &           $0.04 \pm .00$ &  $0.02 \pm .00$ &  $\mathbf{0.05 \pm .00}$ \\
\bottomrule
\end{tabular}
     \caption{The comparison of BBM-RS with other interpretable models (with standard error).}
    \label{tab:bbm_cmp_full}
\end{table*}

\end{document}